\documentclass{article}

\usepackage{amsthm} 
\usepackage{amsmath, amssymb}
\usepackage{breakcites}

\newcommand{\norm}[1]{\left\lVert#1\right\rVert} 
\newtheorem{theorem}{Theorem}
\newtheorem{lemma}{Lemma}
\newtheorem{remark}{Remark}
\newtheorem{definition}{Definition}

\newtheorem{assumption}{Assumption}

\newtheorem{innercustomass}{Assumption}
\newenvironment{customass}[1]
{\renewcommand\theinnercustomass{#1}\innercustomass}
{\endinnercustomass}

\usepackage[accepted]{LNGD}

\usepackage[square]{natbib}
\setcitestyle{numbers, square}

\begin{document}

\onecolumn

\aistatstitle{Generalization Bounds for Label Noise Stochastic Gradient Descent}
\aistatsauthor{ Jung Eun Huh \And Patrick Rebeschini}
\aistatsaddress{ Department of Statistics \\ University of Oxford \And  Department of Statistics \\ University of Oxford}

\begin{abstract}
  We develop generalization error bounds for stochastic gradient descent (SGD) with label noise in non-convex settings under uniform dissipativity and smoothness conditions. Under a suitable choice of semimetric, we establish a contraction in Wasserstein distance of the label noise stochastic gradient flow that depends polynomially on the parameter dimension $d$. Using the framework of algorithmic stability, we derive time-independent generalisation error bounds for the discretized algorithm with a constant learning rate. The error bound we achieve scales polynomially with $d$ and with the rate of $n^{-2/3}$, where $n$ is the sample size. This rate is better than the best-known rate of $n^{-1/2}$ established for stochastic gradient Langevin dynamics (SGLD)---which employs parameter-independent Gaussian noise---under similar conditions. Our analysis offers quantitative insights into the effect of label noise.
\end{abstract}

\section{INTRODUCTION}

One of the central objectives in statistical learning theory is to establish generalization error bounds for learning algorithms to assess the difference between the population risk of learned parameters and their empirical risk on training data. Ever since \citet{Bousquet} unveiled a fundamental connection between generalization error and algorithmic stability, which gauges a learning algorithm's sensitivity to perturbations in training data, numerous studies have used the framework of uniform stability to investigate generalization properties in gradient-based methods, encompassing both convex and non-convex settings, e.g.~\cite{Hardt, London2017, Mou, Li, Feldman, Bassily2020, Lei20c, farnia2021train, Farghly, lei2021generalization, kozachkov2022generalization, zhu2023uniformintime}.

A line of research has focused on understanding the generalization properties resulting from the incorporation of artificial noise into stochastic gradient descent (SGD) methods, as initiated by \citet{Keskar}~\citep{Pensia, Chaudhari, Mou, Negrea, Amir, Wu2022}. Initial studies examined parameter-independent isotropic Gaussian noise, as used in stochastic gradient Langevin dynamics (SGLD)~\citep{Welling, Hardt, Xu, Raginsky, Mou, Negrea, zhang, Li, Chau, Farghly, zhu2023uniformintime}.
There is a growing interest in investigating the structural capabilities induced by parameter-\emph{dependent} noise~\citep{Goyal, Shallue, Blanc, Damian, li2022happens}, where true labels at each iteration are replaced with noisy labels. 

However, generalization error bounds for label noise SGD have not received the same attention as their noise-independent counterparts. Most results in label noise SGD have mainly offered local, asymptotic, and phenomenological insights, without focusing on generalization. Examples include unveiling local implicit bias phenomena by investigating the stability of global minimizers~\citep{Blanc}, providing extensions to global implicit bias for sparsity in quadratically parametrised linear regression models~\citep{Haochen, Damian}, and establishing limiting processes with infinitesimal learning rate for analysing the dynamic of label noise SGD~\citep{li2022happens, Pillaud-Vivien}.

\subsection{Contributions}

In this paper, we develop generalization error bounds for label noise in SGD within non-convex settings and offer a direct comparison with SGLD to emphasize the impact of label noise on generalization. Our analysis employs uniform dissipativity and smoothness assumptions, which are commonly considered in the literature on non-convex sampling and optimization~\citep{Eberle, Raginsky, Xu2, Erdogdu, zhang, Chau, Farghly}.

Under our assumptions, we establish an exponential Wasserstein contraction property for label noise SGD exhibiting a polynomial dependence on the parameter dimension $d$. This contraction property drives the convergence of our generalization error bounds, which also have polynomial dependence on the dimension. Specifically, leveraging a \emph{uniform} dissipativity assumption, we employ the 2-Wasserstein contraction theorem presented in \citet{Wang} to establish the exponential contraction of the Wasserstein distance. This analysis is tailored to a particular semimetric we use for the purpose of analyzing uniform stability. To carry out this analysis, we use the continuous counterpart of the algorithm, known as the stochastic gradient \emph{flow} (SGF). This involves the utilization of Itô calculus and linear algebra techniques to handle the parameter-dependent rectangular matrix noise term.

By leveraging algorithmic stability, we employ our contraction result to establish time-independent generalization error bounds for label noise SGD. Our bounds approach zero as the sample size $n$ increases at the rate of $\mathcal{O}(n^{-2/3})$, achieved by scaling the learning rate as $\mathcal{O}(n^{-2/3})$. This rate is faster than the best-known rate of $\mathcal{O}(n^{-1/2})$ established for SGLD (i.e.\ SGD with parameter-independent Gaussian noise) under similar assumptions \citep{Farghly}, as detailed in the direct comparison in Section 5. The faster decay rate can be established due to the higher dependence of label noise SGD on the learning rate $\eta$, as shown in Table \ref{table:comparison}. This dependence is readily discernible through the presence of the multiplicative factor $\sqrt{\eta}$ in the diffusion part of SGF~\eqref{eq:SDE}, in contrast to the parameter-independent noise flow dynamics of SGLD \eqref{eq:SGLD_SDE}, where the term $\sqrt{\eta}$ is absent.
This dependence has implications for the synchronous-type coupling technique we use to estimate the discretization error. It allows for a more favorable choice of the learning rate---$\mathcal{O}(n^{-2/3})$ instead of $\mathcal{O}(n^{-1/2})$ as seen in SGLD---resulting in a faster generalization error rate.

The bounds we derive for label noise SGD exhibit a reduced dependency on the parameter dimension $d$, in contrast to previous bounds for SGLD \citep{Farghly}. This reduction stems from two factors, as elaborated in Section 5. 
Firstly, the noise term in SGF is dimension-independent, with the Wiener process in \eqref{eq:SDE} being $k$-dimensional, where $k$ denotes the minibatch size. In contrast, the Wiener process in the SGLD flow \eqref{eq:SGLD_SDE} is $d$-dimensional. Secondly, the contraction result we establish under uniform dissipativity has a polynomial dependence on $d$. In contrast, prior results used a weaker form of dissipativity and only established exponential-dependence on $d$.\footnote{The stronger dissipativity assumption we use does not impact the influence of label noise on the relationship with the learning rate $\eta$---our choice of $\eta$ does not depend on $d$---or the rates we establish as a function of $n$.}

Section 2 introduces the framework of algorithmic stability and the assumptions we work with. Section 3 presents our contraction result and generalization error bounds for label noise SGD. Section 4 illustrates the proof schemes with supporting lemmas. Section 5 offers a comparison with prior work on generalization bounds for SGLD \citep{Farghly}. Section 6 is the conclusion. Proofs are in the Appendices.

\section{SETUP AND PRELIMINARIES}\label{sec:setup}

To formalize the learning task, we consider an input-output space $\mathcal{Z} = \mathcal{X} \times \mathcal{Y}$, where $\mathcal{X} \subseteq \mathbb{R}^{p}$ represents the feature space and $\mathcal{Y} \subseteq \mathbb{R}$ represents the label space. The parameter space $\Omega \subseteq \mathbb{R}^d$ contains possible parameters of a data-generating distribution. We have a training dataset $S$ that consists of $n$ sample pairs $z_1, \ldots, z_n \in \mathcal{Z}$, where each pair $z_i = (x_i, y_i)$ is drawn independently from a fixed probability distribution $\mathcal{P}$.

The goal is to learn a non-convex model function $f$ belonging to the family $\mathcal{F}$, where $\theta \in \Omega$ serves as the parameter. Thus, $f(\theta, x_i)$ corresponds to the predicted output for a given input $x_i$ and parameter $\theta$.

Define our loss function $\ell: \Omega \times \mathcal{Z} \rightarrow \mathbb{R}$ of model $f : \Omega \times \mathcal{X} \rightarrow \mathcal{Y}$ as the squared loss:
\begin{align*}
    \ell (\theta, z_i) := \frac{1}{2}(f(\theta,x_i) - y_i)^2.
\end{align*}
We aim to find a parameter $\theta \in \Omega$ that minimizes the \textit{population risk} $L_{\mathcal{P}}$ which is defined by:
\begin{align*}
    L_P(\theta)& := \mathbb{E}_{z \sim \mathcal{P}} [\ell(\theta, z)].
\end{align*}

In settings where the data distribution $\mathcal{P}$ is unknown, calculating the population risk is often infeasible. Hence, we shift our focus to the \textit{empirical risk} $L_S(\theta)$: 
\begin{align}\label{eq:squaredloss}
    L_S ( \theta) := \frac{1}{n} \sum_{i=1}^{n} \ell (\theta, z_i).
\end{align}
The squared loss is convex with respect to the model output, but the non-convexity of $f$ makes the loss function non-convex with respect to the model parameters.

\subsection{Generalization Error Bound via Uniform Stability}\label{sec:stability}

For an algorithm $A$ trained by dataset $S$, we define the \textit{generalization error} to be the difference between the empirical risk and the population risk:
\begin{align*}
    \mathrm{gen}(A) := L_P(A(S)) - L_S(A(S)).
\end{align*}
We bound the generalization error in expectation using the notion of \textit{uniform stability}.

\begin{definition}[Uniform stability {\citep[Definition 2.1]{Hardt}}]\label{def:unifstab}
    A randomized algorithm $A$ is $\varepsilon$-uniformly stable if
    \begin{align*}
        \varepsilon_{stab}(A) \!:=\! \sup_{S \simeq \widehat{S}} \sup_{z \in \mathcal{Z} } \mathbb{E} \left[\ell(A(S), z) \!-\! \ell(A(\widehat{S}), z)\right] \!\leq\! \varepsilon.
    \end{align*}
    The first supremum is over pairs of datasets $S \simeq \widehat{S}$, where $S,\widehat{S} \in \mathcal{Z}^n$ differ by a single element independently drawn from the same data distribution.
\end{definition}

The idea of bounding generalization error by uniform stability was  proposed by \citet{Bousquet} and was further expanded by \citet{Elisseeff} to include random algorithms, with multiple further extensions in the literature. In this paper, we consider the notion of stability introduced in \citet{Hardt}.

\begin{theorem}[Generalization error in expectation~{\citep[Theorem 2.2]{Hardt}}]\label{thm:stabandgen}
    Let $A$ be an $\varepsilon$-uniformly stable algorithm. Then,
    \begin{align*}
        |\mathbb{E}_{A, S} [\mathrm{gen}(A)]| \leq \varepsilon.
    \end{align*}
\end{theorem}

\subsection{Label Noise Stochastic Gradient Descent}\label{sec:algorithm}
Denote \textit{mini-batch average} $L_S(\theta, B)$ as the average of the instance losses $\{\ell (\theta, z_i)\}$ over a uniformly sampled mini-batch $B \subset [n]$ of size $k \leq n$:
\begin{align}\label{eq:minibatchloss}
    \!\!\!L_S(\theta, B) \!:=\! \frac{1}{|B|} \sum_{i \in B} \ell (\theta, z_i) \!=\! \frac{1}{2k} \sum_{i \in B} (f(\theta,x_i) \!-\! y_i)^2.\!
\end{align}
We minimise the training loss in~\eqref{eq:minibatchloss} with {label noise} SGD. Namely, during each gradient step $t > 0$, we explicitly introduce Gaussian random noise $\xi_t \sim \mathcal{N}(0, \delta I_n)$ to the label vector $y = (y_1, \ldots, y_n) \in \mathbb{R}^n$. Define $\widetilde{S} = (\widetilde{z}_1, \ldots, \widetilde{z}_n)$, where $\widetilde{z}_i = (x_i, \widetilde{y}_i) = (x_i, y_i + (\xi_t)_i)$. The update rule of the algorithm started from $\theta_0$ with initial distribution $\mu_0$ corresponds to:
\begin{align} \label{eq:LNGD}
    \theta_{t+1} &= \theta_t - \eta \nabla L_{\widetilde{S}}(\theta_t, B_{t+1}) \notag\\
    &= \theta_t - \eta \nabla L_S(\theta_t, B_{t+1}) +  \frac{\eta}{k} (\nabla \textbf{f}(\theta_t, X_{B_{t+1}}))^\top (\xi_t)_{B_{t+1}}, &\theta_0 \sim \mu_0,
\end{align}
where $\eta >0$ is the learning rate, $(B_t)_{t=1}^{\infty}$ is an i.i.d.\ sequence of uniformly sampled batches of size $k$, $X_{B_{t+1}} \in \mathbb{R}^{k\times p}$ is a submatrix of $X := [x_1^\top, \ldots, x_n^\top]^\top$ with only rows of mini-batch $B_{t+1}$ and $(\xi_t)_{B_{t+1}} \in \mathbb{R}^k$ is also a subvector of $\xi_t$ corresponding to the mini-batch $B_{t+1}$. The matrix $\nabla \textbf{f}(\theta, X_{B_{t+1}}) \in \mathbb{R}^{k\times d}$ consists of model gradients, where $(\nabla \textbf{f}(\theta, X_{B_{t+1}}))_i := \nabla_\theta f(\theta, x_i) \in \mathbb{R}^d$.

In this paper, when the context is clear, we may use $\nabla \textbf{f}$ or $\nabla \textbf{f} (\theta)$ instead of $\nabla \textbf{f}(\theta,X)$. Unless specified, $\nabla$ denotes the gradient with respect to the parameter $\theta$.

\subsubsection{Label Noise Stochastic Gradient Flow}
To understand the label noise SGD dynamics on the non-convex objective in \eqref{eq:minibatchloss}, we use a continuous-time model known as the stochastic gradient flow (SGF). Recent studies have also explored (stochastic) diffusion processes to represent and analyze the dynamics in discrete sequential processes~\citep{Li_SDE, Farghly, Pillaud-Vivien}.

The update rule~\eqref{eq:LNGD} corresponds to the Euler-Maruyama discretization of the stochastic differential equation~(SDE):
\begin{align}\label{eq:SDE}
    d\Theta_t  = - \nabla L_S(\Theta_t, B_{\lceil \frac{t}{\eta} \rceil} ) dt + \frac{\sqrt{\delta \eta}}{|B_{\lceil \frac{t}{\eta} \rceil}|}\! \left( \nabla \textbf{f}(\Theta_t, X_{B_{\lceil \frac{t}{\eta} \rceil}})\right)^\top dW_t,
\end{align}
where $W_t$ is a $k$-dimensional standard Wiener process, $\Theta_t \in \mathbb{R}^d$, $\nabla L_S(\Theta_t) \in \mathbb{R}^d$, and $\delta, \eta, n \in \mathbb{R}$.

The same SDE can be derived for any label noise with zero mean and $\delta I$ covariance through the construction detailed in Appendix~\ref{appendix:sde}. Under smoothness assumptions on the loss function, these SDEs are considered to have strong solutions \citep[Theorem 3.1]{Oksendal}.

As $\theta_t$ in the update rule~\eqref{eq:LNGD} is a Markov process, we define its Markov kernel as $R_\theta$. We will also use the notation $\mu R_\theta^s$ to represent the law of $\theta_{t+s}$ given that $\theta_t$ follows the distribution $\mu$. The process $\Theta_t$ in the SDE~\eqref{eq:SDE} may not necessarily be a continuous-time Markov process due to its dependence on $B_k$. However, the discrete-time process $(\Theta_{t\eta})_{t=0}^\infty$ satisfies the Markov property, and we denote its kernel as $R_\Theta$. We use the notation $\mu P_t^B$ to denote the law of $\lambda_t$, the solution to the SDE with a deterministic batch $B \subset [n]$. Hence, $\mu R_\Theta$ is obtained by integrating $\mu P_\eta^B$ over $B$ with respect to the mini-batch distribution. We use $\widehat{\theta}_t, \widehat{\Theta}_t, \widehat{R}_\theta, \widehat{R}_\Theta$ and $\widehat{P}_t^B$ to denote the corresponding counterparts of $\theta_t, \Theta_t, R_\theta, R_\Theta$ and $P_t^B$ when trained with a perturbed dataset $\widehat{S}$ instead of $S$, where $S$ and $\widehat{S}$ differ in a single element as specified in Definition~\ref{def:unifstab}.

\subsection{Wasserstein Distance}\label{sec:tech}
Algorithmic stability is often measured in terms of $p$-\textit{Wasserstein distance}~\citep{Raginsky, Farghly}, defined as $$W_p (\mu, \nu) := \left(\inf_{\pi \in \mathcal{C}(\mu, \nu)} \int \norm{x-y}^p \pi (dx, dy) \right)^{1/p},$$
where $\norm{\cdot}$ is the Euclidean norm and $\mathcal{C}(\mu, \nu)$ is the set of all couplings of $\mu$ and $\nu$, that is, the set of all probability measure with marginals $\mu$ and $\nu$.

The conventional approach for bounding uniform stability relies on the Lipschitz continuity of the loss function~\citep{Raginsky, Farghly}. Without this continuity, the 2-Wasserstein distance metric becomes insufficient for bounding uniform stability. Following the approach in \citet{Farghly}, we introduce the \emph{semimetric}\footnote{A semimetric is a function defined on $\mathbb{R}^d \times \mathbb{R}^d$ that is symmetric and non-negative with $\rho_g(x, y) > 0$ for $x \neq y$ but is not necessarily satisfying the triangle inequality.}
\begin{align}\label{eq:semimetric}
    \!\rho_g (x,y) := g(\norm{x-y}_2) (1+2\varepsilon+\varepsilon\norm{x}_2^2 +\varepsilon\norm{y}_2^2),\!\!
\end{align}
where $\varepsilon <1$, $g: \mathbb{R}^+ { \cup \{0\}} \rightarrow \mathbb{R}^+ {\cup \{0\}}$ is concave, bounded, and non-decreasing. We consider the \textit{$\rho_g$-Wasserstein distance} based on the semimetric $\rho_g$:
\begin{align}\label{eq:Wasserstein}
    W_{\rho_g} (\mu, \nu) := \inf_{\pi \in \mathcal{C}(\mu, \nu)} \int \rho_g(x,y) \pi(dx, dy).
\end{align}

\subsection{Assumptions}
Our analysis relies on four assumptions we now introduce. The first assumption concerns dissipativity, which is commonly (c.f.\ introduction) imposed to ensure that the diffusion process converges towards the origin rather than diverging when it is far from it, as noted by \citet{Erdogdu}. 

\begin{definition}
    A stochastic process $d\theta_t = b (\theta_t) dt + G(\theta_t) dW_t$ is $\alpha$-uniformly dissipative for $p \in [1, \infty)$ and $\alpha > 0$, if $\forall \theta, \theta' \in \mathbb{R}^d$
    \begin{align}\label{eq:unifdiss}
        2 \langle b(\theta) - b(\theta'), \theta - \theta' \rangle + \norm{ G(\theta) - G(\theta')}_F^2 + (p-2) \norm{G(\theta) - G(\theta')}_{op}^2 \leq -\alpha \norm{\theta - \theta'}_2^2.
    \end{align}
\end{definition}

\begin{assumption}[\textbf{A1}]\label{ass:unifdiss}
    The diffusion process~\eqref{eq:SDE} is $\alpha$-uniformly dissipative for $p=2$.
\end{assumption}

The remaining three assumptions specify conditions for the loss function $\ell$, the model function $f$, and the initial parameter condition $\mu_0$. In particular, we impose \textbf{A\ref{ass:lips}} to ensure the boundedness of our noise term.

\begin{assumption}[\textbf{A2}]\label{ass:smooth}
    For each $z \in \mathcal{Z}$, $\ell(\cdot, z)$ is differentiable and \textit{$M$-smooth}, where $M < \alpha /2$: $\forall \theta_1, \theta_2 \in \mathbb{R}^d$ and $\forall z\in \mathcal{Z}$,
    \begin{align*}
        \norm{\nabla \ell (\theta_1, z) - \nabla \ell(\theta_2, z)} \leq M \norm{\theta_1 - \theta_2}.
    \end{align*}
\end{assumption}

\begin{assumption}[\textbf{A3}]\label{ass:lips}
    For each $z \in \mathcal{Z}$, $\ell(\cdot, z)$ is \textit{$\ell_f$-Lipschitz}: $\forall \theta_1, \theta_2 \in \mathbb{R}^d$ and $\forall x \in \mathcal{X}$,
    \begin{align*}
        |f(\theta_1, x) - f(\theta_2, x)| \leq \ell_f \norm{\theta_1 - \theta_2}.
    \end{align*}
\end{assumption}

\begin{assumption}[\textbf{A4}]\label{ass:moment}
    The initial condition $\mu_0$ of $\theta_0$ has finite fourth moment $\sigma_4$.
\end{assumption}

To simplify the direct application of existing lemmas on dissipativity properties, we employ the notation: $m := \alpha/4$ and $b := \left(1 + 4/(\alpha^2 - 4M^2) \right) \eta_{\max} \delta \ell_f^2 /(2k)$. Here, $\eta_{\max}$ denotes the maximum allowable learning rate as defined in Theorem~\ref{thm:genbound}.

\section{MAIN RESULTS}\label{chapter3}
\subsection{The Wasserstein Contraction Property}\label{sec:contraction}

Following the approach pursued by \citet{Farghly} (c.f.\ Lemma 4.3 in there), we aim to ensure uniform stability by establishing the contraction in the $\rho_g$-Wasserstein distance.
Prior studies, e.g.~\citep{Eberle,Farghly}, have used reflection couplings to establish this contraction under conditions close to uniform dissipativity (\textbf{A\ref{ass:unifdiss}}). However, these studies considered diffusions characterized by constant (i.e.\ parameter-independent) noise terms. In our work, we adapt a more general 2-Wasserstein contraction result from \citet{Wang} to the $\rho_g$-semimetric, thereby establishing exponential $\rho_g$-Wasserstein contraction property for label noise SGD, even with parameter-dependent noise terms.

\begin{theorem}[Wasserstein contraction]\label{thm:contraction}
    Suppose \textbf{A\ref{ass:unifdiss}} and \textbf{A\ref{ass:smooth}} hold. Then there exists a function $g$ such that for any $t \geq 0$ and $1\leq r < a$ we have
    $$W_{\rho_g}(\mu P_t^B, \nu P_t^B) \leq C_1 e^{-\alpha t} W_{\rho_g} (\mu, \nu),$$
    where \begin{align*}
        C_1 :=\;\frac{1}{\varphi a\zeta_r(a)} \bigg(1 + \varepsilon \bigg\{ 2 + 2\sigma_4^{1/2} +2\tilde{c}(2)^{1/2} + \frac{4b}{m} + \frac{2\delta \eta_{\max}}{km} (d +2) \ell_f^2 \bigg\}\bigg),
    \end{align*}
    with $a > 1$ and $(a-1)/a^2 < \zeta_r(a) \leq 1/a$ defined in the proof in Appendix~\ref{appendix:lemma_2-wass}.
    Also, $g$ is constant on $[R, \infty)$ and $\varphi r \leq g(r) \leq r$ for some $R, \varphi \in \mathbb{R}^+$.
\end{theorem}

\begin{remark}[Bound on the contraction coefficient]\label{remark:convergence}
    We can bound $C_1$ as $C_1 \leq e^{\alpha \eta}$ by appropriately selecting $a \equiv a(\alpha, \eta, \varphi, s)$ and $\varepsilon \equiv \varepsilon(\delta, d, m, b, \eta, \ell_f, \sigma_4, k, s)$.  The exponential dependence of $C_1$ on $\eta$ can be restricted by the constraint $\eta \leq \eta_{\max}$. All other constants and parameters involved exhibit only polynomial dependencies. Detailed expressions for the parameters $\varphi, a, \varepsilon$, and $s$ are available in Appendix~\ref{appendix:AlphaandEpsilon}.
\end{remark}

The parameters $\varphi, a, \zeta_r(a), \varepsilon,$ and $\tilde{c}(2)$ in Theorem \ref{thm:contraction} do not depend on sample size. The parameters $\varphi, a, \zeta_r(a),$ and $\tilde{c}(2)$ are independent of dimensionality, with the exception of $\varepsilon$, which we have defined to exhibit a polynomial dependence on the dimension $d$. As a result, our final generalization error bound in Section~\ref{sec:MainResultGenErrBound} preserves its polynomial dependence on dimension.

\begin{remark}[Dimension dependence]
    The primary factor enabling to achieve polynomial dependence on the dimension $d$ is the use of the uniform dissipativity condition \textbf{A\ref{ass:unifdiss}} and the contraction result from \citet{Wang}, rather than the presence of label noise itself.
\end{remark}

\subsection{Generalization Error Bounds}\label{sec:MainResultGenErrBound}
We now derive an upper bound on the expected generalization error $|\mathbb{E}_{A, S} [\mathrm{gen}(A)]|$ for a randomly selected dataset $S$ and a randomized algorithm $A$ which belongs to the class of iterative algorithms described in Section~\ref{sec:algorithm}. The explicit expressions for all parameters can be found in the proof  provided in Appendix~\ref{sec:complete}.

\begin{theorem}[Generalisation error bounds]\label{thm:genbound}
    Suppose \textbf{A\ref{ass:unifdiss}}, \textbf{A\ref{ass:smooth}}, \textbf{A\ref{ass:lips}}, and \textbf{A\ref{ass:moment}} hold and $\eta \leq \eta_{\max} := \min\{\frac{1}{m}, \\\frac{m}{2M^2}\}$. Then, for any $t \in \mathbb{N}$, the continuous-time algorithm attains the generalization error bound
    \begin{align}
        |\mathbb{E} \mathrm{gen}(\Theta_{\eta t})| \leq C_2 \min \left\{\eta t, \frac{n(\eta +2/\alpha)}{(n-k)} \right\} 
        \cdot \; \frac{1}{n}   \left(\frac{\eta}{k^{1/2}}  +\eta^{1/2} + k^{1/2}  +  \frac{k}{\eta^{1/2}} \right).\label{eq:ctsbound}
    \end{align}
    The discrete-time algorithm attains the bound
    \begin{align}
        |&\mathbb{E} \mathrm{gen}(\theta_t)| \leq C_3 \min \left\{\eta t, \frac{n(\eta +2/\alpha)}{(n-k)} \right\} 
        \cdot \bigg[\frac{1}{n}\left\{\frac{\eta}{k^{1/2}}  +\eta^{1/2} + k^{1/2}  +  \frac{k}{\eta^{1/2}} \right\} + \eta  + \frac{\eta}{k^{1/2}} \bigg].\label{eq:discretebound}
    \end{align}
    The positive parameters 
    $C_2 \equiv C_2(\delta, d, m, b, M, \ell_f, \sigma_4, \\\varphi, R, \varepsilon)$ and $C_3 \equiv C_3(\delta, d, m, b, M, \ell_f, \sigma_4, \varphi, R, \varepsilon)$ are given in \eqref{const:C2} and \eqref{const:C3} in Appendix~\ref{sec:complete}.
\end{theorem}

\begin{remark}[Sample size dependence]\label{remark:decayrate}
    Choosing $\eta = \mathcal{O}(n^{-2/3})$ achieves the fastest decaying generalization error bound of $\mathcal{O}(n^{-2/3})$ as derived in Appendix~\ref{appendix:decayrate}.
\end{remark}

\begin{remark}[Dimension dependence]
    Our parameters $C_2$ and $C_3$ exhibit polynomial dependencies on parameter dimension $d$, as well as on $\delta, m, b, M, \ell_f, \sigma_4, \varphi, R, \varepsilon$. The generalization error bounds for both the continuous and the discrete-time algorithm increase at a rate of $d^{5/2}$ as detailed in Table~\ref{table:comparison}. Our bound remains independent of the feature dimension $p$ since the feature vectors only affect our algorithm through the model function $f$, which has an output dimension of 1.
\end{remark}

\begin{remark}[Time independence]
    Following prior works, we adopt the term ``time-independent'' bounds to denote bounds that remain constant as a function of time $t$ for $t$ large enough. See Lemma~\ref{lemma:sgld} for a reference.
\end{remark}

\section{Proof Schemes}\label{sec:proofschemes}

\subsection{Proof for Wasserstein Contraction}
The $\rho_g$-Wasserstein contraction analysis in Theorem~\ref{thm:contraction} builds upon the 2-Wasserstein contraction result in \citet{Wang} under uniform dissipativity. We can draw from Theorem 2.5 in \citet{Wang} that if \textbf{A\ref{ass:unifdiss}} and \textbf{A\ref{ass:lips}} are satisfied, then for any $t \geq 0$, the following contraction property holds:
\begin{align}\label{eq:2contract}
    W_2(\mu P_t^B, \nu P_t^B) \leq e^{-\alpha t /2} W_2(\mu, \nu).
\end{align}
Transition from 2-Wasserstein contraction to $\rho_g$-Wasserstein contraction is achieved by leveraging the properties of the semimetric $\rho_g$ as elaborated in \citet{Farghly} and the Reverse Jensen's Inequality in \citet{wunder2021reverse}.

Refer to Lemma D.3 in \citet{Farghly} for the following inequality, which holds for any two probability measures $\mu$ and $\nu$ on $\mathbb{R}^d$:
\begin{align}\label{eq:rho-2}
    W_{\rho_g} (\mu, \nu) \leq W_2 (\mu, \nu) (1\!+\!2\varepsilon \!+\! \varepsilon\mu(\norm{\cdot}^4)^{\frac{1}{2}} \!+\! \varepsilon\nu(\norm{\cdot}^4)^{\frac{1}{2}}).
\end{align} 
To utilize inequality~\eqref{eq:rho-2}, we derive the moment bound (Lemma~\ref{lemma:moment}) and moment estimate bound (Lemma~\ref{lemma:momentestimate}) tailored for label noise SGD. These derivations require adjustments using It\^o calculus tools to accommodate the parameter-dependent nature of label noise SGD and the non-square matrix form of the noise term.

\begin{lemma}[Moment bound]\label{lemma:moment}
    Suppose \textbf{A\ref{ass:unifdiss}} and \textbf{A\ref{ass:lips}} hold and $\mu$ is a probability measure on $\mathbb{R}^d$. Then, for any $B \subset [n]$, we have
    \begin{align*}
        \mu P_t^B (\norm{\cdot}^p) 	&\leq \mu (\norm{\cdot}^p) + \bigg[\frac{2b}{m} +\frac{\delta \eta}{km} (p+d-2)   \ell_f^2\bigg]^{p/2}.
    \end{align*}
\end{lemma}
Combining the results in \eqref{eq:2contract} and \eqref{eq:rho-2} with the moment bound in Lemma~\ref{lemma:moment}, we obtain the following inequality, which holds under the assumptions \textbf{A\ref{ass:unifdiss}} and \textbf{A\ref{ass:lips}}:
\begin{align*}
    W_{\rho_g} (\mu P_\eta^B, \nu \widehat{P}_\eta^B)
    \!\leq \!e^{-\alpha t / 2} W_2(\mu, \nu) \bigg(\!1 \!+\! \varepsilon \bigg\{2 \!+\! \mu(\norm{\cdot}^2)^{\frac{1}{2}}
    + \nu(\norm{\cdot}^2)^{\frac{1}{2}}\!+\! \frac{4b}{m} \!+\! \frac{2\delta \eta}{km} (d \!+\!2) \ell_f^2\bigg\}\bigg).
\end{align*}
Our analysis focuses on the contraction of $\rho_g$-Wasserstein distance between the distributions $\mu=\mu_0 R_\Theta^t$ and $\nu  = \mu_0 R_\Theta^t$, which represent the laws of our processes when initiated from the same distribution $\mu_0$ and trained with datasets that differ in a single element. We establish the following moment estimate bound by further utilizing the smoothness (\textbf{A\ref{ass:smooth}}) and finite fourth moment (\textbf{A\ref{ass:moment}}) assumptions.

\begin{lemma}[Moment estimate bound]\label{lemma:momentestimate}
    Suppose \textbf{A\ref{ass:unifdiss}}, \textbf{A\ref{ass:smooth}}, \textbf{A\ref{ass:lips}}, and \textbf{A\ref{ass:moment}} hold. Then $$\mu R_\theta^t(\norm{\cdot}^{2p}) \leq\mu (\norm{\cdot}^{2p}) + \tilde{c}(p),$$
    where $\eta < \eta_{\max} := \min\{\frac{1}{m}, \frac{m}{2M^2}\}$ and
    \begin{align*}
        \tilde{c}(p) = \eta_{\max} \big\{\left(3b \right)^p (\eta_{\max} + 2/m)^{p-1} + p (2p-1) \delta \ell_f^2 (\eta_{\max} + 2/m)^{p-2} \left(3b \right)^{p-1} \eta_{\max}^2 +\{p(2p-1)\}^{p+1} \delta^p \ell_f^{2p} \eta_{\max}^{2p-1} \big\}.
    \end{align*}
\end{lemma}

As a consequence of Lemma~\ref{lemma:momentestimate}, we have:
\begin{align*}
    W_{\rho_g} \!(\mu P_\eta^B, \nu \widehat{P}_\eta^B) \!\leq\!&\; e^{-\alpha t / 2} W_2(\mu, \nu) \bigg(1 + \varepsilon \bigg\{ 2 + 2\sigma_4^{1/2} \!\!+2\tilde{c}(2)^{1/2} + \frac{4b}{m} + \frac{2\delta \eta}{km} (d +2) \ell_f^2 \bigg\}\bigg).
\end{align*}
Lastly, we establish the following lemma, based on the Reverse Jensen's Inequality in \citet{wunder2021reverse}.
\begin{lemma}\label{lemma:2-wass}
    There exists a function $g$ constant on $[R, \infty)$ with $\varphi r \leq g(r) \leq r$ for some $R, \varphi \in \mathbb{R}^+$ such that, for any two probability measures $\mu$ and $\nu$ on $\mathbb{R}^d$, we have $$W_2(\mu, \nu) \leq \frac{1}{\varphi a \zeta_b (a)} W_{\rho_g}(\mu, \nu).$$
\end{lemma}

Lemma \ref{lemma:2-wass} guarantees the existence of a function $g$ such that the semimetric $\rho_g$ exhibits the exponential contraction property outlined in Theorem \ref{thm:contraction}.

\subsection{Proof for Generalization Error Bounds}\label{sec:ProofSchemeforGenError}
The proof of Theorem~\ref{thm:genbound} follows the stability framework outlined in Section~\ref{sec:stability}, adhering to the dissipativity and smoothness conditions we consider. A result from \citet{Farghly} shows that a uniform stability bound can be obtained by controlling the $\rho_g$-Wasserstein distance between the laws of algorithms $A(S)$ and $A(\widehat{S})$, where $S \simeq \widehat{S}$.

\begin{lemma}[{\cite[Lemma 4.3]{Farghly}}]\label{Lemmastab}
    Suppose \textbf{A\ref{ass:unifdiss}} and \textbf{A\ref{ass:smooth}} hold and let $A$ be a random algorithm. Then
    \begin{align*}
        \varepsilon_{stab} (A) \!\leq\! \frac{M (b/m \!+\!1)}{\varphi \varepsilon (R \vee 1)} \!\sup_{S \simeq \widehat{S}}\! W_{\rho_g} \!\!\left(law (A(S)), law(A(\widehat{S}))\!\right)\!.\!
    \end{align*}
\end{lemma}

\begin{remark}
    Lemma 4.3 in \citet{Farghly} applies under a weaker assumption of dissipativity and remains valid under our stronger assumption (\textbf{A\ref{ass:unifdiss}}). This connection is elaborated in Section~\ref{sec:comparisonSetting}.
\end{remark}

Hence, it is sufficient to control the quantities $W_{\rho_g}(\mu_0 R_\Theta^t, \mu_0 \widehat{R}_\Theta^t)$ and $W_{\rho_g}(\mu_0 R_\theta^t, \mu_0 \widehat{R}_\theta^t)$ to prove Theorem~\ref{thm:genbound}. Recall from Section~\ref{sec:algorithm} that $\mu R_\Theta$ is obtained by integrating $\mu P_\eta^B$ over $B$ with respect to the mini-batch distribution. To begin, we estimate the divergence $W_{\rho_g} (\mu P_\eta^B, \nu \widehat{P}_\eta^B)$ using the divergence bound (Lemma~\ref{lemma:divergence}), the moment bound (Lemma~\ref{lemma:moment}) and the moment estimate bound (Lemma~\ref{lemma:momentestimate}) with $p=4$.

\label{sec:divergence}
\begin{lemma}[Divergence bound]\label{lemma:divergence}
    Suppose \textbf{A\ref{ass:unifdiss}, A\ref{ass:smooth}}, and \textbf{A\ref{ass:lips}} hold. Then
    \begin{align*}
        \mathbb{E} \norm{\theta_t - \theta_0}^2 \!\leq\! 4M^2 \!\left(\!\mathbb{E}\norm{\theta_0}^2 \!+\!   \frac{3b}{m} \!+\!\frac{\delta \eta d}{k m} \ell_f^2 \!\right)\! t\!+\!\frac{2\delta \eta}{k} \ell_f^2 t.
    \end{align*}
\end{lemma}

This lemma computes the extent to which the process $\theta_t$ deviates from the initial condition $\theta_0$.

Without loss of generality, assume that the datasets $S$ and $\widehat{S}$ differ only at $i^{th}$ element. Considering that $\mathbb{P}(i \in B) = k/n$, the convexity of the $\rho_g$-Wasserstein distance (Lemma~\ref{Lemma2.3} in Appendix~\ref{appendix:semimetric}) gives the following inequality:
\begin{align*}
    W_{\rho_g} (\mu R_\Theta, \nu \widehat{R}_\Theta) \leq \frac{k}{n} \sup_{B: n \in B} W_{\rho_g} (\mu P_\eta^B, \nu \widehat{P}_\eta^B) + \left(1-\frac{k}{n}\right) \sup_{B: n \notin B} W_{\rho_g} (\mu P_\eta^B, \nu \widehat{P}_\eta^B).
\end{align*}

If $i \notin B$, then $\widehat{P}^B = P^B$ so the processes $\Theta_{t\eta}$ and $\widehat{\Theta}_{t\eta}$ contract in $\rho_g$-Wasserstein distance by Theorem~\ref{thm:contraction}. If $i \in B$, the divergence $W_{\rho_g} (\mu P_\eta^B, \nu \widehat{P}_\eta^B)$ obtained above provides uniform bounds on the extent to which $\Theta_{t \eta}$ and $\widehat{\Theta}_{t\eta}$ can deviate from each other. 

Using induction and auxiliary inequalities, we derive a bound for $\varepsilon_{stab}(\Theta_{\eta t})$ in terms of Lemma~\ref{Lemmastab}. By Theorem~\ref{thm:stabandgen}, this bound serves as the generalization error bound for our continuous-time algorithm, as in~\eqref{eq:ctsbound}.
 
So far, we analysed the continuous-time dynamics of our algorithm. The discrete-time process \eqref{eq:LNGD} corresponds to the Euler-Maruyama discretization of \eqref{eq:SDE}. We derive discretization error bounds using synchronous-type couplings between $\theta_\eta$ and $\Theta_{t \eta}$, with both processes sharing the same Brownian motion.

\begin{lemma}[Discretization error bound]\label{lemma:discretization}
    Suppose \textbf{A\ref{ass:unifdiss}, A\ref{ass:smooth}}, and \textbf{A\ref{ass:lips}} hold. Then, for any probability measure $\mu$ on $\mathbb{R}^d$, we have
    \begin{align*}
        W_2(\mu &R_\theta, \mu R_\Theta)^2 \leq 8 \eta^4  \exp \left( 4 \eta^2 M^2 \right) \left[  \frac{2}{3}  M^4 \left( \mu\norm{\cdot}^2 + \frac{b}{m}\right)   + \left( M^2 +2 \right) \frac{\delta }{2k} \ell_f^2\right].
    \end{align*}
\end{lemma}

To extend the generalization error bound established for the continuous-time algorithm~\eqref{eq:ctsbound} to its discrete-time counterpart~\eqref{eq:discretebound}, we add the one-step discretization error to the continuous-time error bound using the weak triangle inequality (Lemma~\ref{Lemma D.1} in Appendix~\ref{appendix:semimetric}).

\section{COMPARISON WITH SGLD}
Label noise SGD is a parameter-dependent noisy algorithm often compared with parameter-independent noisy algorithms like SGLD~\citep{Haochen}. \citet{Farghly} present a discrete-time generalization error bound for SGLD in a dissipative and smooth setting, which decays to zero at a rate of $\mathcal{O}(n^{-1/2})$ with an appropriate learning rate scaling as $\mathcal{O}(n^{-1/2})$. In comparison, our result exhibits a faster rate of decay, as discussed in the introduction. 

\begin{lemma}[{\cite[Theorem 4.1]{Farghly}}]\label{lemma:sgld}
    If $\eta \in (0,1) $ then for any $t \in \mathbb{N}$, the {continuous-time algorithm} attains the {generalization bound}
    \begin{align*}
        |\mathbb{E} \mathrm{gen}(\Theta_{\eta t})| < C_5 \min \left\{\eta t, \frac{(C_4 + 1)n}{n-k}\right\} \frac{k}{n \eta^{1/2}}
    \end{align*}
    Furthermore, if $\eta \leq 1/2m$, then the {discrete-time algorithm attains the generalization bound}
    \begin{align*}
        |\mathbb{E} \mathrm{gen}(\theta_t)| <&\; C_6  \min \left\{\eta t, \frac{(C_4 + 1)n}{n-k}\right\} \left(\frac{k}{n \eta^{1/2}} + \eta^{1/2}\right).
    \end{align*}
    The parameters $C_4, C_5, C_6$  depend on $M, m, b, d, \beta$. Here, $\beta^{-1} > 0$ represents the noise level.
\end{lemma}

\subsection{Comparison of Settings}\label{sec:comparisonSetting}
The proof of the generalization error bound for SGLD in \citet{Farghly} also relies on uniform stability and is built upon largely the same assumptions we use, except for one significant difference that arises in the analytical framework regarding the concept of dissipativity. \citet{Farghly} consider the following assumption in place of the uniform dissipativity assumption~(\textbf{A\ref{ass:unifdiss}}) we use:

\begin{customass}{1$^\prime$}[\textbf{A1$^\prime$}]\label{ass:diss}
    The loss function $\ell(\cdot, z)$ is \textit{$(m,b)$-dissipative}: there exists $m>0$ and $b \geq 0$ such that, for all $\theta \in \mathbb{R}^d$ and $z \in \mathcal{Z}$,
    $$\langle \theta, \nabla \ell(\theta, z)\rangle \geq m\norm{\theta}^2 - b \;\;\; \forall \theta \in \mathbb{R}^d.$$
\end{customass} 

Uniform dissipativity is the key factor that allows our results to circumvent the exponential dependence on the parameter dimension $d$ as established in SGLD's bounds in \citet{Farghly}, leading to  polynomial dependence. However, it is important to stress that the contraction result is unrelated to the dependence of our final generalization error bound the learning rate $\eta$ and sample size $n$. Consequently, the faster decay rate as a function of $n$ highlighted in our generalization error bounds is attributable to the advantages provided by label noise rather than the imposition of uniform dissipativity. This observation is further supported by the following lemma, where we establish a relationship between the assumptions of uniform dissipativity (\textbf{A\ref{ass:unifdiss}}) and the dissipativity (\textbf{A\ref{ass:diss}}).

\begin{lemma}\label{lemma:twodiss}
    Under \textbf{A\ref{ass:smooth}} and \textbf{A\ref{ass:lips}}, the uniform dissipativity assumption \textbf{A\ref{ass:unifdiss}} implies the dissipativity assumption \textbf{A\ref{ass:diss}} with $m=\alpha/4$ and $b = \left(\frac{4}{\alpha^2 - 4M^2} +1 \right) \frac{\eta_{\max}}{2k} \delta \ell_f^2$. The converse holds if $m^3 < M^2 b$ and $\norm{\theta} < B$ for all $\theta$, where $B$ is within the interval
    \begin{align*}
        \frac{1}{2M}\!\left(\!m\!-\!M\!\sqrt{\frac{b}{m}}\!\pm\!\sqrt{\!\left(\!M\!\sqrt{\frac{b}{m}}\!-\!m\!\right)^2\!\!\!\!-\!4M\!\left(\!b\!+\!\frac{\eta\delta\ell_f^2}{k} \!\right)}\;\right).     
    \end{align*}
\end{lemma}

This lemma illustrates that by bounding the parameter space and imposing constraints on dissipativity and smoothness constants, we can treat the analytical framework of \citet{Farghly} and our own as equivalent. This enables a direct comparison between the two algorithms, label noise SGD and SGLD. The proof of Lemma~\ref{lemma:twodiss} is in Appendix~\ref{appendix:unifdiss-diss}.

Regarding the absence of the Lipschitzness assumption on the model (\textbf{A\ref{ass:lips}}) in \citet{Farghly}, it is worth noting the strong connection between \textbf{A\ref{ass:smooth}} and \textbf{A\ref{ass:lips}} in label noise SGD with squared loss $L_S$. Near the global minimizer $\theta^*$, it is observed that $\norm{\frac{1}{k} \nabla \mathbf{f}(\theta^*)^\top \nabla \mathbf{f}(\theta^*)}_2 \approx \norm{\nabla^2 L_S(\theta^*)}_2$, as discussed in \citet{Damian} and \citet{li2022happens}. Thus, under \textbf{A\ref{ass:lips}}, we have the following inequalities for our loss $L_S$:
\begin{align*}
    &\;k\norm{\nabla^2 L_S(\theta^*)}_2 \approx \norm{ \nabla \mathbf{f}^\top \nabla \mathbf{f}}_2 =\sigma_{\max} ( \nabla \mathbf{f}^\top \nabla \mathbf{f})\leq \sum_i^k \lambda_i ( \nabla \mathbf{f}^\top \nabla \mathbf{f})
    =\mathrm{Tr}( \nabla \mathbf{f}^\top \nabla \mathbf{f}) =  \sum_{i=1}^k \norm{\nabla f_i}^2 < k \ell_f^2.
\end{align*}
This confirms the $\ell_f^2$-smoothness of $L_S$ (\textbf{A\ref{ass:smooth}} with $M = \ell_f^2$) near the global minimum.

\subsection{Label Noise and Faster Decay Rate}
We pinpoint the reasons for the faster rate of decay (as a function of the sample size $n$) in the generalization error bound of label noise SGD compared to SGLD by closely examining the differences in the proof components, as outlined in Table~\ref{table:comparison}.

\begin{table}[h]
\caption{Bounds with respect to $\eta$, $d$, and $n$.} \label{table:comparison}
\begin{center}
\begin{tabular}{|c||c|c|} 
    \hline
    Bound Term & SGLD & \begin{tabular}{@{}c@{}}Label Noise \\ SGD \end{tabular} \\
    \hline\hline
    Divergence& $\mathcal{O}(d+1)$ & $\mathcal{O}(\eta d +\eta)$ (\ref{lemma:divergence}) \\ [0.5ex]
    \hline
    Moment& $\mathcal{O}(d^2+1)$ & $\!\!\mathcal{O}(\eta^2 d^2+1)\!\!$ (\ref{lemma:moment}) \\[0.5ex]
    \hline
    Moment estimate& $\!\!\mathcal{O}(d^2\!+\!d\!+\!1)\!\!$ & $\mathcal{O}(1)$ (\ref{lemma:momentestimate}) \\[0.5ex]
    \hline
    Discretization error & \parbox{1cm}{\begin{align*}\mathcal{O}( d \eta^3 e^{\eta^2})\end{align*}} & $\mathcal{O}(\eta^4 e^{\eta^2})$ (\ref{lemma:discretization})\\ [1ex]
    \hline
    $|\mathbb{E} \mathrm{gen}(\theta_{t})|$ & \parbox{1cm}{\begin{align*}\!\!\mathcal{O}\bigg(e^{(d+\sqrt{d})}\left(\frac{d^{7/2}}{n \eta^{1/2}} \!+\! d^{3/2}\eta^{1/2} \right)\bigg)\end{align*}} & \parbox{1cm}{\begin{align*}\;\;\mathcal{O}\bigg(\frac{d^{3/2}}{n}\left[d\eta +(d\eta)^{1/2} + 1 + (d\eta)^{-1/2}\right]\!+ d\eta \bigg)\;\end{align*}} \eqref{thm:genbound}\\ [1ex]
    \hline
\end{tabular}
\end{center}
\end{table}

The noise terms in SGLD and label noise SGD exhibit different dependencies on the learning rate $\eta$, dimension $d$, and batch size $k$. In the update rule of SGLD, the noise term exhibits a square root dependence on the learning rate $\eta$:
$$ \;\theta_{t+1} = \theta_t + \eta \nabla L_S (\theta_t, B_{t+1}) + \sqrt{2 \beta^{-1} \eta} \xi_{t+1}, \quad \!\theta_0 \sim \mu_0.$$

Consequently, the resulting noise term in the associated stochastic process becomes independent of $\eta$ by the derivation detailed in Appendix~\ref{appendix:sde}. The stochastic process associated to SGLD is expressed as:
\begin{align}
    \!\!\!\!\!d\Theta_t \!=\! -\nabla L_S (\Theta_t, B_{\lceil t/\eta \rceil})dt \!+\! \sqrt{2 \beta^{-1}} d\widetilde{W}_t, \Theta_0 \sim \mu_0,\!\!\!\!\!\label{eq:SGLD_SDE}
\end{align}
where $\widetilde{W}_t$ is a $d$-dimensional standard Wiener process. The noise term in this stochastic process is independent of both $\eta$ and $k$.

In contrast, the update rule of label noise SGD~\eqref{eq:LNGD} has a linear dependence of the noise term on $\eta$. This linear relationship arises because label noise impacts the loss function, and its gradient is directly scaled by the learning rate $\eta$ in the update rule. Thus, as shown in Appendix~\ref{appendix:sde}, the noise term in the stochastic process of label noise SGD~\eqref{eq:SDE} is linearly dependent on $\eta/k$. 

The faster decay rate of the label noise SGD bound compared to the SGLD bound is primarily attributed to its discretization error bound. Table~\ref{table:comparison} shows that the SGLD discretization error bound scales as $\mathcal{O}(\eta^3),$ whereas that of label noise SGD scales as $\mathcal{O}(\eta^4),$ a consequence of the synchronous-type coupling method detailed in Appendix~\ref{appendix:discretization}.
The noise term's dependency on model parameters in label noise SGD unavoidably introduces $\eta$-dependent noise in our coupling method, unlike the synchronous-type coupling method used for parameter-independent noise. This, coupled with the divergence bound's dependence on $\eta$, strengthens the dependence of the discretization error bound on $\eta$ and leads to a faster decay rate of our discrete-time generalization bound through an appropriate choice of $\eta$.

\subsection{Dimensionality Dependencies}
A distinguishing trait of the generalization error bound presented in \citet{Farghly} is its exponential dependence on the parameter dimension $d$. This dependence is a consequence of the contraction result employed by the authors under the dissipativity assumption they considered. In contrast, our approach, inspired by the 2-Wasserstein contraction result from \citet{Wang} under uniform dissipativity, allows us to circumvent this dependency, leading our generalization error bound displaying polynomial scaling with the dimension $d$.

Furthermore, the difference in the dimension of the Wiener process, which is $k$-dimensional in label noise SGD~\eqref{eq:SDE} and $d$-dimensional in SGLD~\eqref{eq:SGLD_SDE}, leads to reduced dependence on the parameter dimension within our proof components and, consequently, our generalization bounds. This indicates the advantages offered by label noise. Upon examining the parameters in the proof of Theorem~\ref{thm:genbound}, it is noteworthy that the divergence bound and moment bound of label noise SGD depends on $\eta$, which is different from that of SGLD. This leads to similar scaling of $\eta$ and $d$ in each term of the generalization error bound for label noise SGD, indicating that controlling $\eta$ can alleviate the increase in bounds due to high dimensionality.

\section{CONCLUSION}

The proof technique we employ to establish the contraction property for label noise SGD with polynomial dependence on the dimension $d$ hinges on the assumption of uniform dissipativity. This assumption enables us to avoid the need for reflection coupling, which was utilized in prior research involving parameter-independent noise SGLD \citep{Farghly}. This, in turn, allows us to direct our attention toward understanding the impact of label noise on the selection of learning rate scaling, thereby achieving improved generalization error bounds as a function of the sample size $n$.

We defer the task of establishing results for label noise SGD under a less restrictive form of dissipativity to future research. This pursuit may involve employing Kendall-Cranston couplings~\citep{kendall1986nonnegative, cranston1991gradient} for parameter-dependent noise terms, i.e.\ non-constant diffusion coefficients.

\onecolumn
\aistatstitle{
Supplementary Materials}

\section{Technical Backgrounds}
\subsection{Continuous-time stochastic dynamics modeling}\label{appendix:sde}
A standard formulation of a SDE for the process $(\Theta_t)_{t=0}^\infty$ is given by:
\begin{align*}
	d\Theta_t = b(\Theta_t, t)dt + G(\Theta_t, t) dW_t,
\end{align*}
where $W$ denotes the Wiener process (standard Brownian motion). Here, the term $b(\Theta_t, t)$ is the drift term, which determines the trend or direction of the process, and $G(\Theta_t, t)$ is the noise term, which determines the randomness of the process.

Consider an update rule
\begin{align}\label{eq:update}
	\theta_{t+1} = \theta_t - \eta \nabla L(\theta_t) + \sqrt{\eta} V_t,
\end{align}
where $L: \mathbb{R}^d \rightarrow \mathbb{R}$ is an arbitrary function and $V_t$ is a $d$-dimensional random vector. Let $\Sigma := \frac{1}{\eta} \mathrm{Cov}\left[V_t | \theta_t = \theta \right]$.
Then, the update (\ref{eq:update}) is the Euler-Maruyama discretization of the time-homogeneous SDE
\begin{align*}
	d\Theta_t = - \nabla L(\Theta_t)dt + (\eta \Sigma)^{1/2} dW_t,
\end{align*}
where $W_t$ is a Wiener process.

\subsection{It\^o calculus}

In many papers that analyze diffusion processes of Gaussian noise algorithms, such as SGLD, the noise terms are often in a simple constant scalar form, making calculations relatively straightforward. However, analyzing the diffusion process~\eqref{eq:SDE} of label noise SGD with  squared loss requires more involved calculations due to a more general noise term: a $d \times k$ matrix that depends on the current value of the process.

Hence, we record below some key lemmas of It\^o calculus, which is an extension of calculus to stochastic processes. These lemmas are essential in proving Theorem~\ref{thm:genbound}, particularly when extending the proof of \citet{Farghly} to the label noise SGD algorithm.

\begin{lemma} [It\^o's lemma~\citep{Ito}]\label{Ito}
	Let $\Theta_t$ be a $\mathbb{R}^d$-valued It\^o process satisfying the SDE
	\begin{align*}
		\mathrm{d}\Theta_t = b_t \mathrm{d}t + G_t \mathrm{d}W_t,
	\end{align*}
	where $\mu_t \equiv \mu (\Theta_t, t)$ and $G_t \equiv G(\Theta_t, t)$ are adapted processes to the same filtration as the $n$-dimensional Wiener's process $W_t$. Here, $b_t$ is $\mathbb{R}^d$-valued and $G_t$ is $\mathbb{R}^{d\times n}$-valued.
	
	Suppose that $\phi \in \mathcal{C}^2$. Then, with probability 1, for all $t\geq 0$,
	\begin{align*}
		\mathrm{d}\phi(\Theta_t) = \left\{\frac{\partial \phi}{\partial t} + (\nabla \phi)^\top b_t + \frac{1}{2} \textrm{Tr} [G_t^\top (\nabla^2 \phi) G_t] \right\} \mathrm{d}t + (\nabla \phi)^\top G_t \mathrm{d}W_t.
	\end{align*}
\end{lemma}

\begin{lemma}[It\^o isometry~\citep{Oksendal}]
	If $g(t, w)$ is bounded and elementary then
	\begin{align*}
		\mathbb{E}\left[\left(\int_s^t g(t, w) dW_t\right)^2\right] = \mathbb{E} \left[\int_s^t g(t, w)^2 dt\right].
	\end{align*}
\end{lemma}

\subsection{Relevant inequalities}
Throughout our proofs, we employ valuable inequalities, which are elaborated upon as follows.
\begin{lemma}[Gr\"onwall's lemma~\citep{Gronwall}]
	Assume $\phi: [0, T] \rightarrow \mathbb{R}$ is a bounded non-negative measurable function, $C:[0,T] \rightarrow \mathbb{R}$ is a non-negative integrable function and $B \geq 0$ is a constant with the property that 
	\begin{align*}
		\phi(t) \leq B + \int_0^t C(\tau) \phi(\tau)d\tau \;\;\; \forall t \in [0,T].
	\end{align*}
	Then,
	\begin{align*}
		\phi(t) \leq B \exp \left(\int_0^t C(\tau) d\tau \right) \;\;\; \forall t \in [0,T].
	\end{align*}
\end{lemma}

\begin{theorem}[Young's inequality for products~\citep{Young}]
	If $a \geq 0$ and $b \geq 0$ are non-negative real numbers and if $p > 1$ and $q>1$ are real numbers such that $\frac{1}{p} + \frac{1}{q} = 1$, then
	\begin{align*}
		ab \leq \frac{a^p}{p} + \frac{b^q}{q}.
	\end{align*}
	Equality holds if and only if $a^p = b^q$.
\end{theorem}

\subsection{Properties of the semimetric}\label{appendix:semimetric}

To make this paper self-contained, we will include the lemmas from \citet{Farghly} regarding the semimetric~\eqref{eq:semimetric} and the Wasserstein distance~\eqref{eq:Wasserstein} for our future reference.

The convexity of the Wasserstein distance is a crucial property that plays a central role in our results:
\begin{lemma}[Convexity of the Wasserstein distance {\citep[Lemma 2.3]{Farghly}}]\label{Lemma2.3}
	Suppose that $\rho_g$ is a semimetric and $\mu_1, \mu_2, \nu_1, \nu_2$ are probability measures. Then, for any $r \in [0,1]$,
	\begin{align*}
		W_{\rho_g} (\mu, \nu) \leq r W_{\rho_g} (\mu_1, \nu_1) + (1-r) W_{\rho_g} (\mu_2, \nu_2),
	\end{align*}
	where we define $\mu(dx) = r\mu_1 (dx) + (1-r) \mu_2(dx)$ and $\nu(dx) = r\nu_1(dx) + (1-r) \nu_2 (dx)$.
\end{lemma}

We require the following lemma for computing the divergence bound:
\begin{lemma}[{\citep[Lemma D.3]{Farghly}}]\label{D.3 Farghly}
	Suppose $X, Y, \Delta_x$ and $\Delta_y$ are random variables on $\mathbb{R}^d$, then
	\begin{align*}
		\mathbb{E} \rho_g(X+\Delta_x, Y + \Delta_y) \leq \mathbb{E}\rho_g(X, Y) + \sigma_\Delta^{1/2} (1 +2\varepsilon + 6\varepsilon\sigma^{1/2}),
	\end{align*}
	where we define $\sigma_\Delta := \mathbb{E} \norm{\Delta_x}^2 \vee \mathbb{E}\norm{\Delta_y}^2$ and $\sigma := \mathbb{E} \norm{X}^4 \vee \mathbb{E} \norm{Y}^4 \vee \mathbb{E} \norm{X+\Delta_x}^4 \vee \mathbb{E} \norm{Y+\Delta_y}^4$.
\end{lemma}

To establish the contraction result and simplify the calculation of discretization error in Lemma~\ref{lemma:discretization}, we will simplify the computation by using the 2-Wasserstein distance. In order to do so, we need the following lemma:

\begin{lemma}
	[Comparison with the 2-Wasserstein distance {\citep[Lemma D.2]{Farghly}}]\label{Lemma D.2}
	For any two probability measures $\mu$ and $\nu$ on $\mathbb{R}^d$,
	\begin{align*}
		W_{\rho_g} (\mu, \nu) \leq W_2 (\mu, \nu) (1+2\varepsilon + \varepsilon\mu(\norm{\cdot}^4)^{1/2} + \varepsilon\nu(\norm{\cdot}^4)^{1/2}).
	\end{align*}
\end{lemma}

Finally, we apply the weak triangle inequality in the following lemma to extend the results from continuous-time dynamics to the discrete-time case:
\begin{lemma}
	[Weak triangle inequality {\citep[Lemma D.1]{Farghly}}]\label{Lemma D.1}
	For any $x, y, z \in \mathbb{R}^d$ it holds that,
	\begin{align*}
		\rho_g(x, y) \leq \rho_g(x,z) + 2 \left(1+ \frac{R}{\varphi} (\varepsilon R \vee 1) \right) \rho_g(z, y).
	\end{align*}
\end{lemma}

\subsection{Regularity assumptions}

The assumptions \textbf{A\ref{ass:smooth}-\ref{ass:moment}} are fairly standard in the literature on non-convex optimization. The uniform dissipativity assumption (\textbf{A\ref{ass:unifdiss}}) merit some discussion.

Some may express concerns that the constraint of the dissipativity condition addressed in Lemma A.2 of \citet{Farghly} may confine the process within the absorbing set, potentially simplifying the dynamics of the stochastic gradient flow that we are analyzing. However, it is worth noting that the dissipative assumptions can be enforced through techniques such as weight decay regularization~\citep{Krogh, Raginsky}. This regularization method allows for control over the dynamics of the algorithm and can help ensure that the dissipativity condition is satisfied, while still preserving the overall stability and convergence properties of the algorithm. Thus, the seemingly restrictive nature of the dissipativity condition can be effectively managed through appropriate regularization techniques, adding flexibility to the analysis and applicability of the results.

\subsection{Derivation of Decay Rates}\label{appendix:decayrate}
In Section~\ref{sec:MainResultGenErrBound}, we determine the fastest decay rate for our generalization error bound concerning parameters like $n, d$, and $\eta$. This computation is done by direct optimization as below.

In Remark~\ref{remark:decayrate}, we aim to control the learning rate $\eta$ to achieve the fastest decay rate of the bound which scales as $\mathcal{O}(n^{-1} \{\eta + \eta^{1/2} + 1 + \eta^{-1/2} + \eta\})$ in terms of $\eta$ and $n$. Therefore, we introduce the variable power of $\eta$ as $\eta = \mathcal{O}(n^q)$, and we optimize for the best value of $q$ that results in the fastest decay rate of our bound as $n$ increases.

\section{Missing Proofs}\label{appendix:lemmaproofs}\label{appendixA}

\subsection{Uniform Dissipativity and Dissipativity}\label{appendix:unifdiss-diss}

In the proof, we leverage the implication of \textbf{A\ref{ass:diss}} from \textbf{A\ref{ass:unifdiss}} in Lemma~\ref{lemma:twodiss}, where we set $m := \alpha/4$ and $b := \left(1 + 4/(\alpha^2 - 4M^2) \right) \eta_{\max} \delta \ell_f^2 /(2k)$ . This simplifies the direct application of existing lemmas pertaining to dissipativity properties.

\begin{proof}[Proof of Lemma~\ref{lemma:twodiss}]
    Suppose that the diffusion process~\ref{eq:SDE} is $\alpha$-uniformly dissipative: $\forall \;\; \theta, \theta^\prime$,
    $$ 2 \langle -\nabla L_S(\theta) + \nabla L_S(\theta^\prime), \theta-\theta^\prime \rangle + \frac{\eta \delta}{k^2} \norm{\nabla \mathbf{f}(\theta) - \nabla \mathbf{f}(\theta^\prime)}_F^2 \leq -\alpha \norm{\theta - \theta^\prime}^2.$$

    For $\theta^\prime = 0$,
    \begin{align*}
        &2 \langle -\nabla L_S(\theta) + \nabla L_S(0), \theta \rangle + \frac{\eta \delta}{k^2} \norm{\nabla \mathbf{f}(\theta) - \nabla \mathbf{f}(0)}_F^2 \leq -\alpha \norm{\theta}^2\\
        \Rightarrow\;\; &2 \langle -\nabla L_S(\theta) , \theta \rangle + \frac{\eta \delta}{k^2}  \norm{\nabla \mathbf{f}(\theta)}_F^2 \leq -\alpha \norm{\theta}^2 - 2 \langle \nabla L_S(0), \theta\rangle + \frac{\eta \delta}{k^2} \left( 2 \langle \nabla \mathbf{f}(\theta), \nabla \mathbf{f}(0)\rangle_F - \norm{\nabla \mathbf{f}(0)}_F^2 \right).
    \end{align*}

    By \textbf{A\ref{ass:lips}} and Cauchy-Swartz inequality,
    \begin{align*}
        2 \langle -\nabla L_S(\theta) , \theta \rangle + \frac{\eta \delta}{k^2}  \norm{\nabla \mathbf{f}(\theta)}_F^2 
        &\leq -\alpha \norm{\theta}^2 + 2 \norm{\nabla L_S(0)}\norm{\theta} + \frac{2\eta \delta}{k} \ell_f^2.
    \end{align*}

    Since $2\norm{\nabla L_S (0)}\norm{\theta} \leq \frac{\alpha}{2} \norm{\theta}^2 + \frac{2}{\alpha} \norm{\nabla L_S (0)}^2$,
    \begin{align*}
        2 \langle -\nabla L_S(\theta) , \theta \rangle + \frac{\eta \delta}{k^2}  \norm{\nabla \mathbf{f}(\theta)}_F^2 \leq -\frac{\alpha}{2} \norm{\theta}^2 + \frac{2}{\alpha} \norm{\nabla L_S(0)}^2 + \frac{2\eta \delta}{k} \ell_f^2.
    \end{align*}

    Then, 
    \begin{align*}
         \langle \nabla L_S(\theta), \theta \rangle &\geq \frac{\alpha}{4} \norm{\theta}^2 - \frac{1}{\alpha} \norm{\nabla L_S(0)}^2 - \frac{\eta \delta}{k} \ell_f^2 + \frac{\eta \delta}{2 k^2}  \norm{\nabla \mathbf{f}(\theta)}_F^2\\
        &\geq \frac{\alpha}{4} \norm{\theta}^2 - \frac{1}{\alpha} \norm{\nabla L_S(0)}^2 - \frac{\eta_{\max} \delta}{2k} \ell_f^2\\
        &\geq \frac{\alpha}{4} \norm{\theta}^2 - \left(\frac{4}{\alpha^2 - 4M^2} +1 \right) \frac{\eta_{\max}}{2k} \delta \ell_f^2.
    \end{align*}

    The last inequality used Lemma A.3. in \citet{Farghly}. Hence, \textbf{A\ref{ass:diss}} is satisfied with $m = \alpha / 4$ and $b = \left(\frac{4}{\alpha^2 - 4M^2} +1 \right) \frac{\eta_{\max}}{2k} \delta \ell_f^2.$

    To establish the converse, let's assume, for the sake of contradiction, that for any $\alpha>0$, $\exists \theta, \theta^\prime$ such that $$ 2\langle -\nabla L_S(\theta) + \nabla L_S(\theta^\prime), \theta - \theta^\prime \rangle + \norm{\frac{\sqrt{\eta \delta}}{k} (\nabla \mathbf{f}(\theta) - \nabla \mathbf{f}(\theta^\prime))^\top}_F^2 > -\alpha \norm{\theta - \theta^\prime}_2^2.$$

    For $n=1,2,3, \ldots$, let $\alpha = \frac{1}{n}$. Then, for any $n$, $\exists \theta_n, \theta_n^\prime$ such that $$ 2\langle -\nabla L_S(\theta_n) + \nabla L_S(\theta_n^\prime), \theta_n - \theta_n^\prime \rangle + \norm{\frac{\sqrt{\eta \delta}}{k} (\nabla \mathbf{f}(\theta_n) - \nabla \mathbf{f}(\theta_n^\prime))^\top}_F^2 > -\frac{1}{n} \norm{\theta_n - \theta_n^\prime}_2^2.$$

    Given that $\theta$ is bounded, we can apply the Bolzano-Weierstrass theorem. Consequently, there exists a subsequence $(\theta_{n_s}^\prime)$ that converges to the limit point $u$. As $\norm{\theta - \theta^\prime}^2 < 4B^2  \;\; \forall \theta, \theta^\prime$, for sufficiently large $s$, we have
    \begin{align*}
        & 2\langle -\nabla L_S(\theta_{n_s}) + \nabla L_S(u), \theta_{n_s} - u \rangle + \norm{\frac{\sqrt{\eta \delta}}{k} (\nabla \mathbf{f}(\theta_{n_s}) - \nabla \mathbf{f}(u))^\top}_F^2 \geq 0\\
        \Rightarrow&\;\; \langle -\nabla L_S(\theta_{n_s}), \theta_{n_s} \rangle + \langle \nabla L_S(\theta_{n_s}) , u \rangle - \langle \nabla L_S(u), u \rangle + \langle  \nabla L_S(u), \theta_{n_s}\rangle \geq -\frac{2 \eta \delta}{k} \ell_f^2 \tag{by \textbf{A3}}.
    \end{align*}
    By Cauchy-Swartz,
    \begin{align*}
        & \langle -\nabla L_S(\theta_{n_s}), \theta_{n_s} \rangle + \norm{\nabla L_S(\theta_{n_s})}\norm{u} - \langle \nabla L_S(u), u \rangle + \norm{\nabla L_S(u)}\norm{\theta_{n_s}} \geq -\frac{2 \eta \delta}{k} \ell_f^2 \\
        \Rightarrow&\;\; \langle -\nabla L_S(\theta_{n_s}), \theta_{n_s} \rangle \geq - \norm{\nabla L_S(u)} \norm{\theta_{n_s}} - \norm{\nabla L_S(\theta_{n_s})} \norm{u} + C_u,
    \end{align*}
    where $C_u:= \langle \nabla L_S(u), u \rangle    -\frac{2 \eta \delta}{k} \ell_f^2 \geq m\norm{u} - b - \frac{2 \eta \delta}{k} \ell_f^2$ by \textbf{A\ref{ass:diss}}.

    Then, by \textbf{A\ref{ass:diss}}, 
    $$ -m\norm{\theta_{n_s}} + b \geq \langle -\nabla L_S(\theta_{n_s}), \theta_{n_s} \rangle \geq -\norm{\nabla L_S(u)}  \norm{\theta_{n_s}} - \norm{\nabla L_S(\theta_{n_s})} \norm{u} + C_u$$
    $$\Rightarrow\;\; (\norm{\nabla L_S(u)} -m) \norm{\theta_{n_s}} \geq - \norm{\nabla L_S(\theta_{n_s})} \norm{u} + C_u - b.$$

    Then,
    $$\norm{\nabla L_S(\theta_{n_s})} \geq -\frac{\norm{\nabla L_S(u)} -m}{\norm{u}} \norm{\theta_{n_s}} + \frac{C_u-b}{\norm{u}}. $$

    By \textbf{A\ref{ass:smooth}} and reverse triangle inequality $\norm{\nabla L_S(\theta_{n_s})} - \norm{\nabla L_S(0)} \leq M\norm{\theta_{n_s}}$, so
    $$M\norm{\theta_{n_s}} + \norm{\nabla L_S(0)}\geq \frac{m - \norm{\nabla L_S(u)} }{\norm{u}} \norm{\theta_{n_s}} + \frac{C_u-b}{\norm{u}}$$
    $$\Rightarrow\;\;(2M\norm{u} + \norm{\nabla L_S(0)} -m ) \norm{\theta_{n_s}} \geq C_u - b - \norm{\nabla L_S(0)}\norm{u}. $$

    By Lemma A.3 of \citet{Farghly}, from \textbf{A\ref{ass:diss}} and \textbf{A\ref{ass:smooth}}, above implies
    $$(2M\norm{u} + M \sqrt{b/m} -m ) \norm{\theta_{n_s}} \geq C_u - b - M \sqrt{b/m} \norm{u} $$
    $$\Rightarrow\;\;(2M\norm{u} + M \sqrt{b/m} -m ) \norm{\theta_{n_s}} \geq (m - M \sqrt{b/m})\norm{u} - 2b - \frac{2 \eta \delta}{k} \ell_f^2 $$
    $$\Rightarrow\;\; (m-M\sqrt{b/m})(\norm{u} + \norm{\theta_{n_s}}) \leq 2M\norm{u}\norm{\theta_{n_s}} + 2b + \frac{2\eta\delta}{k}\ell_k^2.$$

    Since $m < M\sqrt{b/m}$ (i.e. $m^3 < M^2 b$),
    \begin{align*}
        & 2MB^2 + 2(M\sqrt{b/m}-m)B + (2b+\frac{2\eta\delta}{k} \ell_f^2) \geq 0 \\
        \iff \;\; & B \notin \Bigg(\frac{1}{2M}\left(m-M\sqrt{b/m} - \sqrt{(M\sqrt{b/m}-m)^2 - 2M(2b+\frac{2\eta\delta}{k} \ell_f^2)}\right)\\
        &\quad\quad\quad , \frac{1}{2M}\left(m-M\sqrt{b/m} + \sqrt{(M\sqrt{b/m}-m)^2 - 2M(2b+\frac{2\eta\delta}{k} \ell_f^2)}\right) \Bigg).
    \end{align*}

    Hence, for
    \begin{align*}
        &B \in \Bigg(\frac{1}{2M}\left(m-M\sqrt{b/m} - \sqrt{(M\sqrt{b/m}-m)^2 - 2M(2b+\frac{2\eta\delta}{k} \ell_f^2)}\right)\\
        &\quad\quad\quad , \frac{1}{2M}\left(m-M\sqrt{b/m} + \sqrt{(M\sqrt{b/m}-m)^2 - 2M(2b+\frac{2\eta\delta}{k} \ell_f^2)}\right) \Bigg),
    \end{align*}
    there is a contradiction hence we have a uniform dissipativity.
\end{proof}

\subsection{Proofs for the Wasserstein Contraction }

\subsubsection{Moment Bound}

\begin{proof}[Proof of Lemma~\ref{lemma:moment}]
	Suppose that $\theta_t$ is a solution of the SDE~\eqref{eq:SDE}. By Ito's Lemma (Lemma \ref{Ito}), for any $\phi \in \mathcal{C}^2$,
	\begin{align*}
		\mathrm{d}\phi = \left\{\frac{\partial \phi}{\partial t} + (\nabla \phi)^\top b_t + \frac{1}{2} \textrm{Tr} [G_t^\top H G_t] \right\} \mathrm{d}t + (\nabla \phi)^\top G_t \mathrm{d}W_t,
	\end{align*}
	with probability 1, where $G_t = \frac{\sqrt{\delta \eta}}{k} \nabla \textbf{f}(\theta_t, X)^\top$, $b_t = -\nabla L_S(\theta_t, B)$ and $H = \nabla^2 \phi$.
	Consider $\phi(\theta) = \norm{\theta}_2^p$. Then,  
	\begin{align*}
		\nabla \phi(\theta) &= p \norm{\theta}_2^{p-2} \theta,\\
		\frac{d\phi}{dt} &= \nabla_{} \phi(\theta)^\top \frac{\mathrm{d}\theta}{\mathrm{d}t}
		= p\norm{\theta}^{p-2} \theta \left[ - \nabla L_S(\theta, B) dt + \frac{\sqrt{\delta \eta}}{k} \nabla \textbf{f}(\theta_t, X)^\top dW_t \right],\\
		H_{ij} &= (\nabla^2 \norm{\theta}_2^p)_{ij} = p\{(p-2)\norm{\theta}_2^{p-4} \theta_i \theta_j \} + p \delta_{ij} \norm{\theta}^{p-2}, \\
		\mathrm{Tr} (H) &= p(p+d-2)\norm{\theta}_2^{p-2}.
	\end{align*}
	We also bound $\textrm{Tr} [G_t^\top H G_t]$ using Von Neumann's Trace inequality~\citep{Mirsky} as below:
	\begingroup
	\allowdisplaybreaks
	\begin{align*}
		\!\!\frac{k^2}{\delta \eta}\mathrm{Tr}[G_t^\top H G_t]\!=& \mathrm{Tr}[\nabla \mathbf{f} H \nabla \mathbf{f}^\top]\!\leq\!  |\mathrm{Tr}[\nabla \mathbf{f} H \nabla \mathbf{f}^\top]|
		\!\leq\! \sum_{i=1}^{d} \sigma_i (\nabla \mathbf{f}^\top \nabla \mathbf{f} ) \sigma_i (H) \tag{by Von Neumann's Trace inequality}\\
		\leq& \sigma_{\max} (\nabla \mathbf{f}^\top \nabla \mathbf{f}) \sum_{i=1}^{d} \sigma_i (H) = \sigma_{\max} (\nabla \mathbf{f}^\top \nabla \mathbf{f}) \mathrm{Tr}(H) \tag{since $H = H^\top$ and $H \succeq 0$}\\
		= &   p(p+d-2)\norm{\theta}_2^{p-2} \sigma_{\max} (\nabla \mathbf{f}^\top \nabla \mathbf{f}) = p(p+d-2)\norm{\theta}_2^{p-2} \norm{\nabla \mathbf{f}^\top \nabla \mathbf{f}}_2.
	\end{align*}
	\endgroup
	Here, $H$ is symmetric since it is Hessian and is positive semi-definite since $\norm{\theta}_2^p$ is convex in $\theta$. ($\norm{\cdot}: \mathbb{R}^n \rightarrow [0, \infty)$ is convex by $\Delta$-inequality and $h: x \mapsto x^p $ is non-decreasing and convex for $[0, \infty)$).
	Since $\nabla \mathbf{f}^\top \nabla \mathbf{f}$ is symmetric positive semidefinite, 
	\begin{align*}
		\norm{\nabla \mathbf{f}^\top \nabla \mathbf{f}}_2 =\sigma_{\max} (\nabla \mathbf{f}^\top \nabla \mathbf{f}) \leq \sum_i \lambda_i (\nabla \mathbf{f}^\top \nabla \mathbf{f})  = \mathrm{Tr}(\nabla \mathbf{f}^\top \nabla \mathbf{f}) =  \sum_{i=1}^k \norm{\nabla f_i}^2 < k \ell_f^2,
	\end{align*}
	where the last inequality follows from \textbf{A\ref{ass:lips}}. Thus, $\mathrm{Tr}[G_t^\top H G_t] \leq	\frac{\delta \eta}{k}  p(p+d-2)\norm{\theta}^{p-2} \ell_f^2$. Then, by It$\hat{\mathrm{o}}$'s lemma,
	\begin{align*}
		\mathrm{d}\norm{\theta_t}^p \leq& 
		 -2 p\norm{\theta_t}^{p-2} \langle \theta_t, \nabla L_S(\theta_t, B) \rangle dt + \frac{\delta \eta}{2 k^2}  p(p+d-2)\norm{\theta_t}^{p-2} \ell_f^2  \mathrm{d}t+2 \frac{\sqrt{\delta \eta}}{k} p\norm{\theta_t}^{p-2} \langle \theta_t, \nabla \textbf{f}^\top \mathrm{d}W_t \rangle.
	\end{align*}
	By $(m,b)$-dissipativity of $L_S$, this can be bounded further as:
	\begin{align*}
		\mathrm{d}\norm{\theta_t}^p \leq &  -2 pm \norm{\theta_t}^{p} dt + p \left\{2b +\frac{\delta \eta}{2 k^2} (p+d-2) \ell_f^2\right\}   \norm{\theta_t}^{p-2} \mathrm{d}t + 2 \frac{\sqrt{\delta \eta}}{k} p\norm{\theta_t}^{p-2} \langle \theta_t, \nabla \textbf{f}^\top \mathrm{d}W_t \rangle\\
		\leq &  -\frac{pm}{2} \norm{\theta_t}^{p} dt + p \left\{b +\frac{\delta \eta}{2 k^2} (p+d-2)   \ell_f^2 \right\}^{p/2}   (m/2)^{1-p/2} \mathrm{d}t+ \frac{\sqrt{\delta \eta}}{k} p\norm{\theta_t}^{p-2} \langle \theta_t, \nabla \textbf{f}^\top \mathrm{d}W_t \rangle,
	\end{align*}
	where for the second inequality we used Young's inequality with exponents $p/(p-2)$ and $p/2$ and $t = \left( \frac{p-2}{2} \left(\frac{m}{2} \right)^{-p/2}\right)^{-2(p-2)/p^2}$.
	Then, by multiplying $e^{pmt/2}$ and using product rule,
	\begin{align*}
		d(e^{pmt/2} \norm{\theta_t}^p) \leq&e^{pmt/2} p \left\{b +\frac{\delta \eta}{2 k^2} (p+d-2) \ell_f^22 \right\}^{p/2}   (m/2)^{1-p/2} \mathrm{d}t + e^{pmt/2} \frac{\sqrt{\delta \eta}}{k} p\norm{\theta_t}^{p-2} \langle \theta_t, \nabla \textbf{f}^\top \mathrm{d}W_t \rangle.
	\end{align*}
	Integrating from $t=0$ to $t=T$,
	\begin{align*}
		\norm{\theta_T}^p \leq & e^{-pmT/2} \norm{\theta_0}^p + (1 - e^{-pmT/2} ) p \left\{b +\frac{\delta \eta}{2 k^2} (p+d-2)    \ell_f^2 \right\}^{p/2}   \frac{2}{pm}  (\frac{m}{2})^{1-p/2}  \\
		&+ e^{-pmT/2} \int_{t=0}^{t=T} e^{pmt/2} \frac{\sqrt{\delta \eta}}{k} p\norm{\theta_t}^{p-2} \langle \theta_t, \nabla \textbf{f}^\top \mathrm{d}W_t \rangle.
	\end{align*}
	Taking expectations,
	\begin{align*}
		\mathbb{E} \norm{\theta_T}^p \leq
		e^{-pmT/2} \mathbb{E}\norm{\theta_0}^p + (1 - e^{-pmT/2} )  \{\frac{2b}{m} +\frac{\delta \eta}{k^2 m} (p+d-2)   \ell_f^2 \}^{p/2}.
	\end{align*}
	Hence,
	\begin{align*}
		\mu P_t^B (\norm{\cdot}^p) \leq& \mu (\norm{\cdot}^p) e^{-pmt/2} + \left[\frac{2b}{m} +\frac{\delta \eta}{k^2 m} (p+d-2)    \ell_f^2\right]^{p/2} ( 1- e^{-pmt/2})\\
		\leq &  \mu (\norm{\cdot}^p) + \left[\frac{2b}{m} +\frac{\delta \eta}{k^2 m} (p+d-2)   \ell_f^2\right]^{p/2},
	\end{align*}
	as required.
\end{proof}

\subsubsection{Moment Estimate Bound}
In order to perform our estimations in a continuous-time setting, we introduce an auxiliary continuous-time process. First, recall the stochastic differential equation (SDE) of label noise gradient descent (LNGD):
\begin{align}\label{sde}
	d\theta_t = - \nabla L_S(\theta_t, B) dt + \frac{\sqrt{\delta \eta}}{k} \left( \nabla \textbf{f}(\theta_t, X_B)\right)^\top dW_t, \;\;\;\;\; \theta_0 \sim \mu_0,
\end{align}
where $(W_t)_{t\geq0}$ is a standard $k$-dimensional Wiener process.

It is worth noting that the SDE~\eqref{sde} has a unique solution on $\mathbb{R}^+$, since the smoothness assumption~(\textbf{A\ref{ass:smooth}}) holds for $L_S$. Hence, we define, for each $\eta >0$, a convenient time-changed version of $\Theta_t$ as $$\theta_t^\eta := \theta_{\eta t}.$$Then, $\tilde{W}_t^\eta := W_{\eta t} / \sqrt{\eta}$ is also a Wiener process and
$$
	d\theta_t^\eta = -\eta L_S(\theta_t^\eta, B) dt+  \frac{\eta \sqrt{\delta}}{k} \left( \nabla \textbf{f}(\theta_t^\eta, X_B)\right)^\top d\tilde{W}_t^\eta, \;\;\;\;\; \theta_0^\eta \sim \mu_0.
$$
We proceed with the required moment estimate that is essential for the derivation of the main result in Theorem~\ref{thm:contraction}.  These estimates will also enable us to calculate how far the process $\theta_t^\eta$ diverges from its initial condition in one step in the derivation of Theorem~\ref{thm:genbound}. To prove these bounds, we will heavily rely on the auxiliary process defined above and perform significant calculations.

\begin{proof}[Proof of Lemma~\ref{lemma:momentestimate}]
	First, note that by It\^o's isometry and commutativity of trace operator with expectation and integral,
	\begin{align*}
		\mathrm{Tr}\left(\mathrm{Var}\left(\int_{u=s}^t \frac{\eta \sqrt{\delta }}{k} ( \nabla \textbf{f}(\theta_u))^\top dW_u\right)\right) 
		&= \frac{\delta \eta^2}{k^2} \mathrm{Tr} \left( \mathbb{E}  \left[\int_{s}^t \nabla \mathbf{f}(\theta_u)^\top  \nabla \mathbf{f}(\theta_u)  du \right] \right)\\
		&=\frac{\delta \eta^2}{k^2} \mathbb{E}  \left[\int_{s}^t  \mathrm{Tr} \left( \nabla \mathbf{f}(\theta_u)^\top  \nabla \mathbf{f}(\theta_u) \right) du \right] \\
		&= \frac{\delta \eta^2}{k^2} \mathbb{E}  \left[\int_{s}^t  \sum_{i=1}^n \norm{\nabla f_i (\theta_u)}^2 du \right]\leq  \frac{\delta \eta^2}{k} \ell_f^2 (t-s).
	\end{align*}
	For any $s \in \mathbb{N}$ and $t \in (s, s+1]$, define $\Delta_{s,t} = \theta_s - \eta \nabla L_S (\theta_s, B)(t-s)$. Note that for a vector $v$, $\mathbb{E}\norm{v}^2 = \norm{\mathbb{E}(v)}^2 + \mathrm{Tr}(\mathrm{Var}(v))$. Then, for $t \in (s, s+1]$,
	\begin{align*}
		\mathbb{E} [ \norm{\theta_t^\eta}^2 \mid \theta_s^\eta] &= \mathbb{E} [ {\theta_t^\eta}^\top \theta_t^\eta \mid \theta_s^\eta] = \mathbb{E} [\mathrm{Tr}({\theta_t^\eta}^\top \theta_t^\eta)  \mid \theta_s^\eta] = \mathbb{E} [\mathrm{Tr}(\theta_t^\eta {\theta_t^\eta}^\top)  \mid \theta_s^\eta] = \mathrm{Tr}(\mathbb{E}[\theta_t^\eta {\theta_t^\eta}^\top \mid \theta_s^\eta]) \\
		&= \mathbb{E}[\theta_t^\eta  \mid \theta_s^\eta]^\top \mathbb{E}[\theta_t^\eta  \mid \theta_s^\eta] +\mathrm{Tr}\left(\mathrm{Var}\left(\int_{u=s}^t \frac{\sqrt{\delta \eta}}{k} ( \nabla \textbf{f}(\theta_u))^\top dW_u\right)\right) \\
		&\leq \norm{\mathbb{E}[\theta_t^\eta  \mid \theta_s^\eta]}^2 + \frac{\delta \eta^2}{k} \ell_f^2 (t-s)
		= \norm{\Delta_{s,t}}^2 + \frac{\delta \eta^2}{k} \ell_f^2 (t-s)).
	\end{align*}
	Here, $\forall \eta \leq \eta_{\max} := \min \{\frac{1}{m}, \frac{m}{2M^2}\}$, 
	\begin{align}
		\norm{\Delta_{s,t}}^2 =&\; \norm{\theta_s^\eta}^2 - 2\eta(t-s) \langle \theta_s^\eta, \nabla L_S(\theta_s^\eta, B) \rangle + \eta^2 \norm{\nabla L_S (\theta_s^\eta, B) (t-s)}^2 \notag\\
		 \leq &\;(1-2m\eta(t-s))\norm{\theta_s^\eta}^2 + 2b\eta (t-s)  + \eta^2 (t-s)^2  \norm{\nabla L_S(\theta_s^\eta, B)}^2 \tag{by \textbf{A\ref{ass:diss}}}\\
		\leq &\;(1-2m\eta (t-s))\norm{\theta_s^\eta}^2 +2b\eta (t-s) + 2\eta^2 (t-s)^2 \left\{M^2 \norm{\theta_s^\eta}^2 + \norm{\nabla L_S(0, B)}^2 \right\} \tag{by \textbf{A\ref{ass:smooth}}}\\
		\leq &\;(1-2m\eta (t-s))\norm{\theta_s^\eta}^2 + 2b\eta (t-s) + 2\eta^2 (t-s)^2 \left\{ M^2 \norm{\theta_s^\eta}^2 + \frac{ M^2 b}{m} \right\} \tag{by Lemma A.3 in \cite{Farghly}}\\
		\leq &\;(1-m\eta (t-s))\norm{\theta_s^\eta}^2 + 2b \eta (t-s) \left(1+ \eta_{\max}  \frac{ M^2 }{m} \right) \label{delta}\\
		\leq&\;(1-m\eta (t-s))\norm{\theta_s^\eta}^2 + 6b \eta (t-s) \notag,
	\end{align}
	where the fourth inequality is from that $2 \eta^2 (t-s)^2 M^2 \leq 2 \eta^2 (t-s)M^2 \leq m \eta(t-n)$.
	For higher moments, the computation is more complex. To simplify the calculation, let $U_{s,t}^\eta := \frac{\eta\sqrt{\delta}}{k} \int_s^t \nabla \mathbf{f} (\theta_r^\eta)^\top d\tilde{W}_r^\eta$ be defined. Then, for $t \in  [s, s+1)$,
	\begingroup
	\allowdisplaybreaks
	\begin{align*}
		\mathbb{E}[|\theta_t^\eta|^{2p} \mid \theta_s^\eta]  =&\; \mathbb{E}\left[\norm{\theta_s^\eta - \eta \nabla L_S(\theta_s^\eta, B)(t-s) + U_{s,t}^\eta}^{2p} \mid  \theta_s^\eta \right]
		= \mathbb{E}\left[\norm{\Delta_{s,t} + U_{s,t}^\eta}^{2p} \mid  \theta_s^\eta \right]\\
		\leq&\; \norm{\Delta_{s,t}}^{2p} + 2p \mathbb{E}\left[  \norm{\Delta_{s,t}}^{2p-2} \langle \Delta_{s,t},U_{s,t}^\eta \rangle \mid  \theta_s^\eta \right] + \sum_{k=2}^{2p} {2p \choose k} \mathbb{E}\left[ \norm{\Delta_{s,t}}^{2p-k} \norm{U_{s,t}^\eta }^k \mid  \theta_s^\eta \right] \tag{by Lemma A.3 in \cite{Chau}}\\
		\leq&\; \norm{\Delta_{s,t}}^{2p} + p(2p-1) \mathbb{E} \left[(\norm{\Delta_{s,t}} +\norm{U_{s,t}^\eta})^{2p-2} \norm{U_{s.t}^\eta}^2 \mid \theta_s^\eta \right] \tag{as in the proof of Lemma 3.9 in \cite{Chau}}\\
		\leq&\; \norm {\Delta_{s,t}} ^{2p} + p (2p - 1) \cdot \frac{\delta \eta^2}{k} \ell_f^2 (t-s) \norm {\Delta_{s,t}} ^{2p-2} + p(2p-1) \mathbb{E} [\norm{ U_{s,t}^\eta} ^{2p}]\\
		\leq&\; \norm{\Delta_{s,t}}^{2p} + p (2p - 1) \cdot \frac{\delta \eta^2}{k} \ell_f^2 (t-s) \norm{ \Delta_{s,t}}^{2p-2}+ \{p(2p-1)\}^{p+1} (t-s)^p \frac{\delta^p \eta^{2p}}{k^{p}} \ell_f^{2p} \tag{by Theorem 7.1 in \cite{Mao}}.
	\end{align*}
	\endgroup
	Note the following inequality for further analysis
	\begin{align}\label{ineq}
		(r+s)^p \leq (1+\varepsilon)^{p-1} r^p + (1+\varepsilon^{-1})^{p-1} s^p,
	\end{align}
	where $p \geq 2, r, s, \geq 0 $ and $\varepsilon >0$. Letting $\varepsilon = m\eta (t-s) /2$,
	\begin{align*}
		\norm{\Delta_{s,t}}^{2p} \leq&\; \left[(1-m\eta (t-s)) \norm{\theta_s^\eta}^2 +3b  \eta (t-s) \right]^p \tag{by~\eqref{delta}}\\
		\leq&\; \left(1 +\frac{m\eta (t-s)}{2}\right)^{p-1} (1-m\eta (t-s))^p \norm{\theta_s^\eta}^{2p} + \left(1 + \frac{2}{m\eta (t-s)}\right)^{p-1} \eta^p (t-s)^p \left(3b \right)^p \tag{by~\eqref{ineq}}\\
		\leq&\; a_{s,t}^{\eta,p} \norm{\theta_s^\eta}^{2p} + b_{s,t}^{\eta, p},
	\end{align*}
	where $a_{s,t}^{\eta,p} = (1 - m\eta (t-s)/2)^{p-1} (1-m\eta (t-s))$ and $ b_{s,t}^{\eta, p} =(\eta (t-s)+ 2/m)^{p-1} \eta (t-s) \left(3b \right)^p$.
	Substituting it yields
	\begin{align*}
		\mathbb{E}[|\theta_t^\eta|^{2p} \mid \theta_s^\eta] \leq&\; a_{s,t}^{\eta,p} \norm{\theta_s^\eta}^{2p} + b_{s,t}^{\eta, p} + p (2p - 1) \cdot \frac{\delta \eta^2}{k} \ell_f^2 (t-s) \left[a_{s,t}^{\eta,p-1} \norm{\theta_s^\eta}^{2(p-1)} + b_{s,t}^{\eta, p-1}\right]\\
		&+ \{p(2p-1)\}^{p+1} (t-s)^p \frac{\delta^p \eta^{2p}}{k^{p}} \ell_f^{2p}.
	\end{align*}
	Define $\widetilde{M}(p) = \sqrt{\frac{4p(2p-1) \delta \eta \ell_f^2}{mk}}$. Then, for $\norm{\theta_s^\eta} \geq \widetilde{M}(p)$, 
	\begin{align*}
		\frac{m\eta (t-s)}{4} \norm{\theta_s^\eta}^{2p} \geq p(2p-1) \frac{\delta \eta^2}{k} \ell_f^2 (t-s) \norm{\theta_s^\eta}^{2(p-1)}.
	\end{align*}
	Hence,
	\begin{align*}
		\mathbb{E}[|\theta_t^\eta|^{2p} \mid \theta_s^\eta] \leq&\;  (1-m\eta(t-s)/4) a_{s,t}^{\eta, p-1} \norm{\theta_s^\eta}^{2p}+ b_{s,t}^{\eta, p} + \eta (t-s) p(2p-1) \frac{\delta \eta}{k} \ell_f^2 b_{s,t}^{\eta, p-1}\\
		&+ \eta^p (t-s)^p \{p(2p-1)\}^{p+1} \frac{\delta^p \eta^p}{k^p} \ell_f^{2p}\\
		\leq&\; (1-m\eta (t-s)) \norm{\theta_s^\eta}^{2p} + \eta(t-s) M(p,\eta, k)
		\leq\; \norm{\theta_s^\eta}^{2p} + \eta M(p,\eta, k),
	\end{align*}
	where 
	\begin{align*}
		M(p, \eta, k) :=&\; (\eta (t-s)+ 2/m)^{p-1} \left(3b \right)^p + \eta (t-s) p (2p-1) \frac{\delta \eta}{k} \ell_f^2 (\eta (t-s)+ 2/m)^{p-2} \left(3b \right)^{p-1}\\
		&+ \eta^{p-1} (t-s)^{p-1} \{p(2p-1)\}^{p+1} \frac{\delta^p \eta^p}{k^p} \ell_f^{2p}\\
		\leq&\;  (\eta + 2/m)^{p-1} \left(3b \right)^p + \eta^2 p (2p-1) \frac{\delta}{k} \ell_f^2 (\eta + 2/m)^{p-2} \left(3b \right)^{p-1}+ \eta^{2p-1} \{p(2p-1)\}^{p+1} \frac{\delta^p}{k^p} \ell_f^{2p}\\
		=:&\; \frac{1}{\eta}\tilde{c}(p).
	\end{align*}
	Similarly, for $\norm{\theta_s^\eta} < \widetilde{M}(p)$ we attain
	\begin{align*}
		\mathbb{E}[|\theta_t^\eta|^{2p} \mid \theta_s^\eta] \leq&\; \norm{\theta_s^\eta}^{2p} + \tilde{c}(p).
	\end{align*}
	Hence, we have
	\begin{align*}
		\mathbb{E}[|\theta_t^\eta|^{2p} \mid \theta_s^\eta] \leq&\; \norm{\theta_s^\eta}^{2p} + \tilde{c}(p),
	\end{align*}
	as required.
\end{proof}

\subsubsection{$\rho_g$-Wasserstein Distance and 2-Wasserstein Distance}
\begin{proof}[Proof of Lemma~\ref{lemma:2-wass}]\label{appendix:lemma_2-wass}
		Let $f: \mathbb{R}_0^+ \rightarrow \mathbb{R}_0^+$ be a function such that $f(x) = \sqrt{x}$ in its domain. Then, since $f$ is concave and $f(0) = 0$, we can apply Lemma 1 in \citet{wunder2021reverse}. Let \begin{align*}
			r := r(p) := \frac{\mathbb{E}[g(\norm{X-Y}^p)]}{\mathbb{E}[g(\norm{X-Y})]^p} \geq 1
		\end{align*}
		be the ratio of first and $p$-th non-centralized moment and \begin{align*}
			\zeta_r (a) := \sup_{\frac{1}{p} + \frac{1}{q} = 1} \frac{[1-r(p)^{\frac{1}{p}}a^{-\frac{1}{q}}]^+}{a}.
		\end{align*}
		Then, $\forall \pi \in \mathcal{C}(\mu, \nu) \;\; \forall{1 \leq r < a}$, \begin{align*}
			\left(\int \norm{x-y}^2 \pi(dx, dy)\right)^{1/2} &\leq \frac{1}{\varphi} \left( \int g(\norm{x-y})^2 \pi(dx, dy) \right)
			= \frac{1}{\varphi a \zeta_b (a)} \cdot \left( \int a g(\norm{x-y})^2 \pi(dx, dy)\right)^{1/2} \cdot \zeta_b (a)\\
			& \leq \frac{1}{\varphi a \zeta_b (a)} \cdot \sup_{s>t} \left( \int s g(\norm{x-y})^2 \pi(dx, dy)\right)^{1/2} \cdot \zeta_t (s)\\
			& \leq  \frac{1}{\varphi a \zeta_b (a)} \cdot \int g(\norm{x-y}) \pi(dx,dy)
			  \leq  \frac{1}{\varphi a \zeta_b (a)} \cdot \int \rho_g(x,y) \pi(dx,dy).
		\end{align*}
	Hence, \begin{align*}
		W_2 (\mu, \nu) := \inf_{\pi \in \mathcal{C}(\mu,\nu)} \left(\int \norm{x-y}^2 \pi(dx,dy)\right)^{1/2}
		\leq  \frac{1}{\varphi a \zeta_b (a)} \cdot \inf_{\pi \in \mathcal{C}(\mu, \nu)} \int \rho_g(x.y) \pi(dx, dy)
		=  \frac{1}{\varphi a \zeta_b (a)} \cdot W_{\rho_g} (\mu, \nu).
        \end{align*}

        Note that since $r(p) \geq 1,$ we have $$\inf_{\frac{1}{p} + \frac{1}{q} = 1} r(p)^{1/p} a^{-1/q} \geq \inf_{\frac{1}{p} + \frac{1}{q} = 1} a^{-1/q} \geq \inf_q a^{-1/q} \geq 0.$$ 
        Then, we can obtain an upper bound for $\zeta_r (a)$ as below:
        \begin{align*}
            \zeta_r (a) &= \frac{1}{a} \sup_{\frac{1}{p} + \frac{1}{q} = 1} \left[1 - r(p)^{1/p} a^{-1/q}\right]^+= \frac{1}{a} \max \left\{0, 1-\inf_{\frac{1}{p} + \frac{1}{q} = 1} r(p)^{1/p} a^{-1/q}\right\}
            \leq \frac{1}{a}.
        \end{align*}
        Similarly, we can obtain a lower bound as:
        \begin{align*}
            \zeta_r (a) &= \frac{1}{a} \sup_{\frac{1}{p} + \frac{1}{q} = 1} \left[1 - r(p)^{1/p} a^{-1/q}\right]^+ > \frac{1}{a} \left[1 - r(p)^0 a^{-1}\right]^+ = \frac{a-1}{a^2}.
        \end{align*}
	\end{proof}

\subsection{Proofs for the Generalization Error Bound}
\subsubsection{Divergence Bound}
\begin{proof}[Proof of Lemma~\ref{lemma:divergence}]
	Integrating the SDE from $0$ to $t$, 
	\begin{align*}
		\theta_t - \theta_0 = - \int_0^t \nabla L_S(\theta_s, B) ds + \int_{0}^t \frac{\sqrt{\delta \eta}}{k} ( \nabla \textbf{f}(\theta_s))^\top dW_s.
	\end{align*}
	Then, by Jensen's inequality of integrals,
	\begin{align*}
		\norm{\theta_t - \theta_0}^2 
		&\leq 2  \int_0^t \norm{\nabla L_S(\theta_s, B)}^2 ds +2 \norm{\int_{s=0}^t \frac{\sqrt{\delta \eta}}{k} ( \nabla \textbf{f}(\theta_s))^\top dW_s}^2.& 
	\end{align*}
	Note that for a vector $v$, $\mathbb{E}\norm{v}^2 = \norm{\mathbb{E}(v)}^2 + \mathrm{Tr}(\mathrm{Var}(v))$. 
	Then, by It\^o's isometry and commutativity of trace operator with expectation and integral, the expectation is
	\begin{align*}
		\mathbb{E} \norm{\theta_t - \theta_0}^2 &\leq 2\int_0^t \mathbb{E} \norm{\nabla L_S(\theta_s, B)}^2 ds + 2 \mathrm{Tr}\left(\mathrm{Var}\left(\int_{s=0}^t \frac{\sqrt{\delta \eta}}{k} ( \nabla \textbf{f}(\theta_s))^\top dW_s\right)\right) \\
		&= 2  \int_0^t \mathbb{E} \norm{\nabla L_S(\theta_s, B)}^2 ds +  \frac{2\delta \eta}{k^2} \mathrm{Tr} \left( \mathbb{E}  \left[\int_{0}^t \nabla \mathbf{f}(\theta_s)^\top  \nabla \mathbf{f}(\theta_s)  ds \right] \right)\\
		& = 2 \int_0^t \mathbb{E} \norm{\nabla L_S(\theta_s, B)}^2 ds + \frac{2\delta \eta}{k^2} \mathbb{E}  \left[\int_{0}^t  \mathrm{Tr} \left( \nabla \mathbf{f}(\theta_s)^\top  \nabla \mathbf{f}(\theta_s) \right) ds \right] \\
		&= 2  \int_0^t \mathbb{E} \norm{\nabla L_S(\theta_s, B)}^2 ds + \frac{2\delta \eta}{k^2} \mathbb{E}  \left[\int_{0}^t  \sum_{i=1}^n \norm{\nabla f_i (\theta_s)}^2 ds \right],
	\end{align*}
	where the last equality is from:
	\begin{align*}
		\mathrm{Tr}(\nabla \mathbf{f}^\top \nabla \mathbf{f})  = \mathrm{Tr}(\sum_{i=1}^n \nabla f_i \nabla f_i^\top) = \sum_{i=1}^n \mathrm{Tr}( \nabla f_i \nabla f_i^\top) = \sum_{i=1}^n \norm{\nabla f_i}^2.
	\end{align*}
	Using Lemma A.3 in \citet{Farghly}, we can find the upper bound of the first term,
	\begin{align*}
		\mathbb{E} \norm{\nabla L_S(\theta_s, B)}^2 &\leq 2\mathbb{E}\norm{\nabla L_S(\theta_s, B) - \nabla L_S(0, B)}^2 + 2\mathbb{E}\norm{\nabla L_S(0, B)}^2\\
		&\leq 2M^2 \norm{\theta_s}^2 + 2M^2 \frac{b}{m} & \tag{by Lemma A.3 in \cite{Farghly}}\\
		&\leq 2M^2 (\mathbb{E}\norm{\theta_0}^2 +  2M^2 \left[\frac{3b}{m} +\frac{\delta \eta d}{k^2 m}  \norm{\nabla \mathbf{f}(\theta_s)^\top \nabla \mathbf{f}(\theta_s)}\right]. & \tag{by Lemma~\ref{lemma:moment}}
	\end{align*} 	
	Hence, we get
	\begin{align*}
		\mathbb{E} \norm{\theta_t - \theta_0}^2 \leq 4M^2  \left[\mathbb{E}\norm{\theta_0}^2 +   \frac{3b}{m} +\frac{\delta \eta d}{k^2 m}  \norm{\nabla \mathbf{f}(\theta_t)^\top \nabla \mathbf{f}(\theta_t)}\right] t +\frac{2\delta \eta}{k^2} \mathbb{E}  \left[\int_{0}^t  \sum_{i=1}^n \norm{\nabla f_i (\theta_s)}^2 ds \right].
	\end{align*}
\end{proof}

\subsubsection{Discretization Error Bound}\label{appendix:discretization}
This section aims to derive the discretization error bounds using synchronous-type coupling. Synchronous coupling is a way to pair samples from two probability distributions that is commonly used to estimate the error between continuous-time and discrete-time stochastic processes. Given two probability measures $\mu$ and $\nu$ on a common space $\mathcal{W}$, a synchronous coupling of $\mu$ and $\nu$ is a joint probability measure $\pi$ on $\mathcal{W} \times \mathcal{W}$ such that the marginals of $\pi$ are $\mu$ and $\nu$, respectively, and $\pi(w,w')=0$ whenever $w \neq w'$. This means that in a synchronous coupling, each sample drawn from $\mu$ is always paired with a sample drawn from $\nu$ in a one-to-one manner.

Recall from Section~\ref{sec:algorithm} that we denote $R_\theta$ to be a Markov kernel of our algorithm $(\theta_t)_{t=0}^\infty$ and $R_\Theta$ to be a Markov kernel of our discrete-time process $(\Theta_{t\eta})_{t=0}^\infty$. Our main objective here is to estimate the Wasserstein distance between $\mu R_\theta$ and $\mu R_\Theta$, where $\mu$ represents an arbitrary probability measure. We focus on obtaining bounds on the Wasserstein distance between two distributions: $\mu R_\theta^B$, which is the distribution of one step of a label noise SGD with fixed mini-batch $B$, and $\mu P_\eta^B$. To do this, we define a coupling $(\tilde{\theta}_t, \lambda_{\eta t})$  for $t \in [0,1]$, 
\begin{align*}
	d\lambda_t = - \nabla L_S(\lambda_t, B)dt + \frac{\sqrt{\delta \eta}}{k} \left( \nabla \textbf{f}(\lambda_t, X_B)\right)^\top dW_t, &\quad\quad \lambda_0 \sim \mu,\\
	\tilde{\theta}_t = \tilde{\theta}_0 - \nabla L_S(\tilde{\theta}_0, B) \eta t +  \int_0^{\eta t} \frac{\sqrt{\delta \eta}}{k} \left( \nabla \textbf{f}(\tilde{\theta}_0, X_B)\right)^\top dW_{s},&\quad\quad \tilde{\theta}_0 = \lambda_0.
\end{align*}
Then, by the convexity of the Wasserstein distance (Lemma~\ref{Lemma2.3}),
\begin{align*}
	W_{\rho_g} (\mu R_\theta, \mu R_\Theta) = {n \choose k}^{-1} \sum_{B \subset [n], |B| = k} W_{\rho_g} (\mu R_\theta^B, \mu P_\eta^B),
\end{align*}
so we derive a bound for Wasserstein distance $W(\mu R_\theta, \mu R_\Theta)$ as in the following lemma.

\begin{proof}[Proof of Lemma~\ref{lemma:discretization}]
	Integrating from $s=0$ to $\eta t$,
	\begin{align*}
		\lambda_{\eta t} = \lambda_0 - \int_0^{\eta t} \nabla L_S(\lambda_s, B) ds + \int_0^{\eta t}  \frac{\sqrt{\delta \eta}}{k} \left( \nabla \textbf{f}(\lambda_s, X_B)\right)^\top dW_s.
	\end{align*}
	Then, by change of variable, 
	\begin{align*}
		\lambda_{\eta t} - \tilde{\theta}_t&= - \eta \int_0^{t} \nabla L_S(\lambda_{\eta s}, B) - \nabla L_S(\tilde{\theta}_0, B) ds +\int_0^{\eta t}  \frac{\sqrt{\delta \eta}}{k} \left( \nabla \textbf{f}(\lambda_s, X_B) - \nabla \textbf{f}(\tilde{\theta}_0, X_B) \right)^\top dW_s.
	\end{align*}
	So, by Jensen's inequality and It\^o's isometry as in the proof of Lemma~\ref{lemma:divergence},
	\begin{align*}
		\mathbb{E}\norm{\lambda_{\eta t} - \tilde{\theta}_t}^2\leq& \; 2\eta^2 \int_0^{t} \mathbb{E}\norm{\nabla L_S(\lambda_{\eta s}, B) - \nabla L_S(\tilde{\theta}_0, B)}^2 ds\\
		&+ 2 \frac{\delta \eta}{k^2} \mathbb{E} \int_0^{\eta t}  \mathrm{Tr} \left[\left( \nabla \textbf{f}(\lambda_s, X_B) - \nabla \textbf{f}(\tilde{\theta}_0, X_B) \right)^\top \left( \nabla \textbf{f}(\lambda_s, X_B) - \nabla \textbf{f}(\tilde{\theta}_0, X_B) \right)\right] ds
		\\
		\leq&\; 2\eta^2 M^2 \int_0^{t} \mathbb{E}\norm{\lambda_{\eta s} - (\tilde{\theta}_s - \tilde{\theta}_s) - \tilde{\theta}_0}^2 ds +2 \frac{\delta \eta}{k^2} \mathbb{E}\left[ \int_0^{\eta t} 4 k \ell_f^2 ds\right]
		\\
		\leq&\;4 \eta^2 M^2 \int_0^{t} \mathbb{E}\norm{\lambda_{\eta s} - \tilde{\theta}_s}^2 ds + 4 \eta^2 M^2 \int_0^{t} \mathbb{E}\norm{\tilde{\theta}_s - \tilde{\theta}_0}^2 ds + \frac{8\delta \eta^2 t }{k} \ell_f^2,
	\end{align*}
	where the second inequality is from \textbf{A\ref{ass:smooth}, A\ref{ass:lips}}. We can also bound the second term by \textbf{A\ref{ass:smooth}, A\ref{ass:lips}} and Lemma A.3 in \citet{Farghly} as
	\begin{align*}
		\mathbb{E} \norm{\tilde{\theta}_s - \tilde{\theta}_0}^2 \leq&\; 2 \eta^2 s^2 \mathbb{E} \norm{\nabla L_S(\tilde{\theta}_0, B)}^2  +  2 \frac{\delta \eta}{k^2} \mathbb{E} \norm{\int_0^{\eta s}  \left( \nabla \textbf{f}(\tilde{\theta}_u, X_B)\right)^\top dW_{u}}^2 \\
		\leq&\; 2 \eta^2 s^2 \mathbb{E} \left\{2M^2 \norm{\tilde{x}_0}^2 + 2M^2\frac{b}{m}\right\} +  2 \frac{\delta \eta}{k^2} \mathbb{E} \left[\int_0^{\eta s} \mathrm{Tr} \left(\nabla \mathbf{f}(\tilde{\theta}_u, X_B)^\top  \nabla \mathbf{f}(\tilde{\theta}_u, X_B) \right) du\right]\\
		\leq&\; 4 \eta^2 s^2 M^2 \left\{ \mathbb{E}\norm{\tilde{x}_0}^2 + \frac{b}{m}\right\}  +  \frac{2 \delta \eta^2 s}{k} \ell_f^2.
	\end{align*}
	Applying Gr\"onwall's inequality with $\phi(t) = \mathbb{E}\norm{\lambda_{\eta t} - \tilde{\theta}_t}^2$, we have
	\begin{align*}
		\;\mathbb{E}{\norm{\lambda_{\eta t} - \tilde{\theta}_t}}^2 &\leq 
		8 \eta^2  \exp \left(4 \eta^2 M^2 t \right) \left[  \eta^2 M^2 \left\{\frac{2}{3} t^3 M^2 \left( \mu(\norm{\cdot}^2) + \frac{b}{m}\right)  +  \frac{ \delta t^2 }{2k} \ell_f^2\right\}  + \frac{\delta \eta^2 t}{k} \ell_f^2\right] \\
		& = 8 \eta^4  \exp \left(4 \eta^2 M^2 t \right) \left[  \frac{2}{3}  M^4  t^3 \left( \mu(\norm{\cdot}^2) + \frac{b}{m}\right)   + (  M^2 t^2+ 2  t)\frac{\delta }{2k} \ell_f^2\right] .
	\end{align*}
	For $t = 1$, 
	\begin{align*}
		\mathbb{E}\norm{\lambda_\eta - \tilde{\theta}_1}^2 &\leq 
		8\eta^4  \exp \left( 4 \eta^2 M^2 \right) \left[  \frac{2}{3} M^4 \left( \mu(\norm{\cdot}^2) + \frac{b}{m}\right)   + ( M^2 +2)\frac{\delta  }{2k} \ell_f^2\right] .
	\end{align*}
\end{proof}

\begin{remark}
	The bound for the discrete-time algorithm is obtained by adding the one-step discretization error to the bound for the continuous-time algorithm.
\end{remark}

\subsubsection{Completing the Proof of Theorem~\ref{thm:genbound}}\label{sec:complete}

With all the necessary components established, we are now able to finalize the proof of Theorem~\ref{thm:genbound}.

\begin{proof}[Proof of Theorem~\ref{thm:genbound}]
	Using Lemma~\ref{D.3 Farghly}, we have the following inequality:
	\begin{align*}
		W_{\rho_g} (\mu P_\eta^B, \nu \widehat{P}_\eta^B) \leq W_{\rho_g}(\mu, \nu) + \tau_\Delta ^{1/2} (1+2\varepsilon+6\varepsilon\tau^{1/2}).
	\end{align*}
	
	Here, $\tau_\Delta$ and $\tau$ are calculated using the divergence bound (Lemma~\ref{lemma:divergence}) and the moment bound (Lemma~\ref{lemma:moment}) with $p=4$, which are given by:
	\begin{align*}
		\tau_\Delta &:= \mathbb{E} \norm{\theta_\eta - \theta_0}^2 \vee \mathbb{E}\norm{\widehat{\theta}_\eta - \widehat{\theta}_0}^2 = 4M^2 \left[\sigma_\Delta^{1/2} +   \frac{3b}{m} +\frac{\delta \eta d}{k m}  \ell_f^2 \right] \eta+\frac{2\delta }{k} \ell_f^2 \eta^2,\\
		\tau &:= \mathbb{E} \norm{\theta_0}^4 \vee \mathbb{E} \norm{\widehat{\theta}_0}^4 \vee \mathbb{E} \norm{\theta_\eta}^4 \vee \mathbb{E} \norm{\widehat{\theta}_\eta}^4 = \sigma_\Delta + \left[\frac{2b}{m} +\frac{\delta \eta}{km} (d+2)   \ell_f^2\right]^{2},
	\end{align*}
	where $\sigma_\Delta = \mu(\norm{\cdot}^4) \vee  \nu(\norm{\cdot}^4)$. Then we can further proceed as:
	\begin{align*}
		&W_{\rho_g} (\mu P_\eta^B, \nu \widehat{P}_\eta^B) \\
        &\leq W_{\rho_g}(\mu, \nu) + 2\eta^{1/2} \left[M^2 \left(\sigma_\Delta^{1/2} +   \frac{3b}{m} +\frac{\delta \eta d}{k m}  \ell_f^2 \right) +\frac{\delta \eta}{2k} \ell_f^2  \right]^{1/2} \cdot \left[1 +2\varepsilon + 6\varepsilon \sigma_\Delta^{1/2} + 6\varepsilon  \left(\frac{2b}{m} +\frac{\delta \eta}{km} (d+2)  \ell_f^2 \right) \right]\\
		&\leq W_{\rho_g}(\mu, \nu) + \left[\eta^{1/2} \left\{4M^2 \left(\sigma_\Delta^{1/2} +   \frac{3b}{m}\right)\right\}^{1/2} + \frac{\eta }{k^{1/2}} \left\{ \left(\frac{2d}{m} M^2 + 1 \right) \ell_f^2 \delta  \right\}^{1/2} \right]\\
		&\quad\cdot \left[1 +2\varepsilon + 6\varepsilon  \left(\sigma_\Delta^{1/2} + \frac{2b}{m} \right)+ 6\varepsilon \frac{\eta}{k} \left(\frac{\delta (d+2) }{m} \ell_f^2 \right) \right]\\
		&\leq W_{\rho_g}(\mu, \nu) + \frac{\eta^2}{k^{3/2}} \left[6\varepsilon\frac{\delta (d+2) \ell_f^2}{m}  \left\{ \left(\frac{2d}{m} M^2 + 1 \right) \ell_f^2 \delta  \right\}^{1/2}  \right]+\frac{\eta^{3/2}}{k}\left[6\varepsilon\frac{\delta (d+2) \ell_f^2}{m}  \left\{4M^2 \left(\sigma_\Delta^{1/2}+   \frac{3b}{m}\right)\right\}^{1/2} \right]\\
        &\quad+\frac{\eta }{k^{1/2}} \left[ \left\{ \left(\frac{2d}{m} M^2 + 1 \right) \ell_f^2 \delta  \right\}^{1/2} \left\{1 +2\varepsilon + 6\varepsilon  \left(\sigma_\Delta^{1/2} + \frac{2b}{m} \right) \right\} \right]\\
        &\quad+ \eta^{1/2} \left[ \left\{4M^2 \left(\sigma_\Delta^{1/2} +   \frac{3b}{m}\right)\right\}^{1/2}\left\{1 +2\varepsilon + 6\varepsilon  \left(\sigma_\Delta^{1/2} + \frac{2b}{m} \right) \right\}\right].
	\end{align*}
	
	Assuming  $\mu = \mu_0 R_\Theta^t$ and $\nu = \mu_0 \widehat{R}_\Theta^t$ for some $t$, we can use the moment estimate bound (Lemma~\ref{lemma:momentestimate}) to obtain $\sigma_\Delta \leq \mu_0(\norm{\cdot}^4) + \tilde{c}(2) = \sigma_4 + \tilde{c}(2)$.
	Therefore, we have:
	\begin{align*}
		W_{\rho_g} (\mu P_\eta^B, \nu \widehat{P}_\eta^B) \leq  W_{\rho_g}(\mu, \nu) + \tilde{c}_1 \frac{\eta^2}{k^{3/2}}  + \tilde{c}_2 \frac{\eta^{3/2}}{k} + \tilde{c}_3 \frac{\eta }{k^{1/2}}  +  \tilde{c}_4 \eta^{1/2},
	\end{align*}
	with parameters
	\begin{align*}
		\tilde{c}_1 &:= 6\varepsilon\frac{\delta (d+2) \ell_f^2}{m}  \left\{ \left(\frac{2d}{m} M^2 + 1 \right) \ell_f^2 \delta  \right\}^{1/2}, \\
		\tilde{c}_2 &:= 6\varepsilon\frac{\delta (d+2) \ell_f^2}{m}  \left\{4M^2 \left( \sigma_4^{1/2} + \tilde{c}(2)^{1/2} +   \frac{3b}{m}\right)\right\}^{1/2},\\
		\tilde{c}_3 &:= \left\{ \left(\frac{2d}{m} M^2 + 1 \right) \ell_f^2 \delta  \right\}^{1/2} \left\{1 +2\varepsilon + 6\varepsilon  \left( \sigma_4^{1/2} + \tilde{c}(2)^{1/2}+ \frac{2b}{m} \right) \right\}, \\
		\tilde{c}_4 &:= \left\{4M^2 \left( \sigma_4^{1/2} + \tilde{c}(2)^{1/2} +   \frac{3b}{m}\right)\right\}^{1/2}\left\{1 +2\varepsilon + 6\varepsilon  \left( \sigma_4^{1/2} + \tilde{c}(2)^{1/2} + \frac{2b}{m} \right) \right\},
	\end{align*}
	where
	\begin{align*}
		\tilde{c}(2) =&\;  \left\{ \frac{18b^2}{m} + 9b^2 \eta_{\max} + 18  b \delta \ell_f^2  \eta_{\max}^2+216 \delta^2 \ell_f^{4} \eta_{\max}^{3} \right\} \eta_{\max}.
	\end{align*}
	
	Without loss of generality, assume that the datasets $S$ and $\widehat{S}$ differs only at $i^{th}$ element. Considering that $\mathbb{P}(i \in B) = k/n$, the convexity of $\rho_g$-Wasserstein distance (Lemma~\ref{Lemma2.3}) gives the following inequality:
	\begin{align*}
		W_{\rho_g} (\mu R_\Theta, \nu \widehat{R}_\Theta) &\leq \frac{k}{n} \sup_{B: n \in B} W_{\rho_g} (\mu P_\eta^B, \nu \widehat{P}_\eta^B) + \left(1-\frac{k}{n}\right) \sup_{B: n \notin B} W_{\rho_g} (\mu P_\eta^B, \nu \widehat{P}_\eta^B)\\
		&\leq \tilde{c}_5 W_{\rho_g}(\mu, \nu) + \frac{1}{n} \left[\tilde{c}_1 \frac{\eta^2}{k^{1/2}}  + \tilde{c}_2 \eta^{3/2} + \tilde{c}_3 \eta k^{1/2}  +  \tilde{c}_4 \eta^{1/2} k \right],
	\end{align*}
	where $\tilde{c}_5 := \frac{k}{n} + \left( 1-\frac{k}{n} \right) C_1 e^{-\alpha t} $. In the second inequality, the second term is bounded by the contraction result in Theorem~\ref{thm:contraction}. By appropriate choice of $\varepsilon$ as denoted in Section~\ref{appendix:AlphaandEpsilon}, we have $\tilde{c}_5 <1$, thus by induction,
	\begin{align}
		W_{\rho_g} (\mu_0 R_\Theta^t, \mu_0 \widehat{R}_\Theta^t) \leq \frac{1-\tilde{c}_5^t}{1-\tilde{c}_5} \cdot  \frac{1}{n} \left[\tilde{c}_1 \frac{\eta^2}{k^{1/2}}  + \tilde{c}_2 \eta^{3/2} + \tilde{c}_3 \eta k^{1/2}  +  \tilde{c}_4 \eta^{1/2} k \right].\label{eq:induction}
	\end{align}
        
	Applying Lemma~\ref{Lemmastab}, we obtain a bound for uniform stability as shown below:
	\begin{align*}
		\varepsilon_{stab} (\Theta_{\eta t}) &\leq \frac{M (b/m +1)}{\varphi \varepsilon (R \vee 1)} \cdot \frac{1-\tilde{c}_5^t}{1-\tilde{c}_5} \cdot \frac{1}{n} \left[\tilde{c}_1 \frac{\eta^2}{k^{1/2}}  + \tilde{c}_2 \eta^{3/2} + \tilde{c}_3 \eta k^{1/2}  +  \tilde{c}_4 \eta^{1/2} k \right]\\
		&\leq C_2 \frac{1-\tilde{c}_5^t}{1-\tilde{c}_5} \cdot \frac{1}{n} \left[\frac{\eta^2}{k^{1/2}}  +\eta^{3/2} + k^{1/2}\eta +  \eta^{1/2} k \right],
	\end{align*}
	where
	\begin{align}\label{const:C2}
		C_2 :=\frac{M (b/m +1)}{\varphi \widetilde{\varepsilon} (R \vee 1)} (\tilde{c}_1 \vee \tilde{c}_2 \vee \tilde{c}_3 \vee \tilde{c}_4).
	\end{align}

        Here, we define $\varepsilon$ to be independent of $\eta$ by bounding $1/\varepsilon$. Recall the choice of $\varepsilon$ from \eqref{eq:epsilon}, then we have 
        \begin{align*}
            \frac{1}{\varepsilon}& \leq \frac{\left({ 2 + 2\sigma_4^{1/2} +2\tilde{c}(2)^{1/2} + \frac{4b}{m} + \frac{2\delta \eta_{\max}}{km} (d +2) \ell_f^2}\right)}{(1+s)e^{-\alpha \eta / 4} -1} 
            \leq \frac{\left({ 2 + 2\sigma_4^{1/2} +2\tilde{c}(2)^{1/2} + \frac{4b}{m} + \frac{2\delta \eta_{\max}}{km} (d +2) \ell_f^2}\right)}{e^{\alpha \eta_{\max} / 4} -1} := \frac{1}{\widetilde{\varepsilon}}.
        \end{align*}
	
	By approximation $1-\tilde{c}_5^t \leq 1 \wedge (1-\tilde{c}_5)t$ and using the bound $(1-e^{-x})^{-1} \leq 1 + 1/x$ 
    (since $e^x \geq 1+x$),
	\begin{align*}
		\frac{1-\tilde{c}_5^t}{1-\tilde{c}_5} \leq \left(\frac{1}{1-\tilde{c}_5} \wedge t\right)=\left(\frac{1}{(1-k/n)(1-C_1 e^{-\alpha \eta})} \wedge t \right) \leq \left(\frac{1}{(1-k/n)(1-e^{-\alpha \eta /2})} \wedge t \right) \leq \left(\frac{n(1+2/\alpha \eta)}{(n-k)} \wedge t\right),
	\end{align*}
        where the second inequality holds because $C_1 \leq e^{\alpha \eta / 2}$ as determined by the selections of $\phi, a,$ and $\varepsilon$ in \eqref{eq:epsilon}.
	
	So, if $\eta \leq \min\{\frac{1}{m}, \frac{m}{2M^2}\}$ then for any $t \in \mathbb{N}$, the continuous-time algorithm attains the generalization error bound
	\begin{align*}
		|\mathbb{E} \mathrm{gen}(\Theta_{\eta t})| \leq C_2 \min \left\{\eta t, \frac{n(\eta +2/\alpha)}{(n-k)} \right\} \frac{1}{n}   \left[\frac{\eta}{k^{1/2}}  +\eta^{1/2} + k^{1/2}  +  \frac{k}{\eta^{1/2}} \right].
	\end{align*}
	
	The result is extended to the discrete-time generalization error bound using the weak triangle inequality (Lemma~\ref{Lemma D.1}),
	\begin{align}\label{ineq:W}
		W_{\rho_g}(\mu_0 R_\theta , \mu_0 \widehat{R}_\theta) \leq \tilde{c}_6 W_{\rho_g} (\mu_0 R_\theta, \mu_0 R_\Theta) + W_{\rho_g} (\mu_0 R_\Theta, \mu_0 \widehat{R}_\Theta) + \tilde{c}_6 W_{\rho_g} (\mu_0 \widehat{R}_\Theta, \mu_0 \widehat{R}_\theta),
	\end{align}
	with 
	\begin{align}\label{const:c6}
		\tilde{c}_6 := 1 + \frac{2 g(R)}{\varphi} (\varepsilon R \vee 1).
	\end{align}
	
	To bound the first and third terms in terms of the 2-Wasserstein distance, we utilize the discretization error bound (Lemma~\ref{Lemma D.2}).
	\begin{align*}
		W_{\rho_g} (\mu_0 R_\theta, \mu_0 R_\Theta)^2 \leq&\; W_2 (\mu_0 R_\theta, \mu_0 R_\Theta)^2(1+2\varepsilon +\varepsilon \mu_0 R_\theta(\norm{\cdot}^4)^{1/2} +\varepsilon \mu_0 R_\Theta(\norm{\cdot}^4)^{1/2}) \\
		\leq&\; W_2 (\mu_0 R_\theta , \mu_0 R_\Theta)^2 (1+2\varepsilon (1+(\sigma_4 + \tilde{c}(2))^{1/2} )) \tag{Lemma~\ref{lemma:momentestimate}}\\
		\leq&\; 8\eta^4 \exp(4\eta^2 M^2) \bigg[\frac{2}{3} M^4 \left(\sigma_4^{\frac{1}{2}} + \tilde{c}(2)^{\frac{1}{2}} + \frac{b}{m}\right) \!+\! \frac{(M^2+ 2)\delta}{2k} \ell_f^2\bigg]  \cdot \left(1+2\varepsilon \left(1+\sigma_4^{\frac{1}{2}} + \tilde{c}(2)^{\frac{1}{2}}\right)\right)^2, \quad \quad \tag{Lemma~\ref{lemma:discretization}}
	\end{align*}
	so that
	\begin{align*}
		W_{\rho_g} &(\mu_0 R_\theta, \mu_0 R_\Theta)\\
		\leq&\; 2\sqrt{2} \eta^2 \exp(2\eta^2 M^2) \bigg[\frac{2}{3} M^4 (\sigma_4^{1/2} + \tilde{c}(2)^{1/2} + \frac{b}{m}) + ( M^2  + 2)\frac{\delta}{2k} \ell_f^2\bigg]^{1/2} \cdot \left(1+2\varepsilon \left(1+\sigma_4^{\frac{1}{2}} + \tilde{c}(2)^{\frac{1}{2}}\right)\right) \\
		\leq&\; 2\sqrt{2} \eta^2 \exp(2\eta^2 M^2 ) \bigg[\frac{2}{3}M^4 \sigma_4^{1/2} + \frac{2}{3}M^4 \tilde{c}(2)^{1/2} + \frac{2bM^4 }{ 3m} + \frac{\delta M^2}{ 2k}\ell_f^2 + \frac{\delta}{k} \ell_f^2\bigg]^{1/2} \cdot \left(1+2\varepsilon \left(1+\sigma_4^{\frac{1}{2}} + \tilde{c}(2)^{\frac{1}{2}}\right)\right) \\
		\leq&\; \exp(2\eta^2 M^2 ) \left[\tilde{c}_7  \eta^2+ \tilde{c}_8\frac{\eta^2}{\sqrt{k}}\right],
	\end{align*}
	with parameters
	\begin{align*}
		\tilde{c}_7& := 2\sqrt{2}  M\left[\frac{2}{3}M^2 \sigma_4^{1/2} +  \frac{2}{3}M^2 \tilde{c}(2)^{1/2}+ \frac{2bM^2}{3m}\right]^{1/2}\cdot \left(1+2\varepsilon \left(1+\sigma_4^{\frac{1}{2}} + \tilde{c}(2)^{\frac{1}{2}}\right)\right), \\
		\tilde{c}_8& := 2 \sqrt{\delta}\left(M  + \sqrt{2}  \right)\ell_f \cdot \left(1+2\varepsilon \left(1+\sigma_4^{\frac{1}{2}} + \tilde{c}(2)^{\frac{1}{2}}\right)\right).
	\end{align*}
	Therefore, the inequality~\eqref{ineq:W} can be rewritten as
	\begin{align*}
		W_{\rho_g}(\mu_0 R_\theta , \mu_0 \widehat{R}_\theta) \leq&\; 2\tilde{c}_6 \exp(2\eta^2 M^2 ) \left[ \tilde{c}_7  \eta^2+ \tilde{c}_8  \frac{\eta^2}{\sqrt{k}}\right]+ W_{\rho_g} (\mu_0 R_\Theta, \mu_0 \widehat{R}_\Theta)\\
		\leq&\; \tilde{c}_5 W_{\rho_g}(\mu, \nu) + \frac{1}{n} \left[\tilde{c}_1 \frac{\eta^2}{k^{1/2}}  + \tilde{c}_2 \eta^{3/2} + \tilde{c}_3 \eta k^{1/2}  +  \tilde{c}_4 \eta^{1/2} k \right]+ 2\tilde{c}_6 \exp(2\eta^2 M^2 )\left[ \tilde{c}_7   \eta^2+ \tilde{c}_8 \frac{\eta^2}{\sqrt{k}} \right].
	\end{align*}
	
	Now, applying the same arguments as above,
	\begin{align*}
		\varepsilon_{stab} (x_t) \leq&\;  C_3 \min \left\{\eta t, \frac{n(\eta +2/\alpha)}{(n-k)} \right\}  \cdot \bigg[ \frac{1}{n}\left\{\frac{\eta}{k^{1/2}}  +\eta^{1/2} + k^{1/2}  +  \frac{k}{\eta^{1/2}} \right\} + \left\{\eta  + \frac{\eta}{k^{1/2}}\right\}\bigg],
	\end{align*}
	where
	\begin{align}\label{const:C3}
		C_3 :=\frac{M (b/m +1)}{\varphi \widetilde{\varepsilon} (R \vee 1)}  (\tilde{c}_1 \vee \tilde{c}_2 \vee \tilde{c}_3 \vee \tilde{c}_4 \vee 2\tilde{c}_6\tilde{c}_7 \vee 2\tilde{c}_6\tilde{c}_8 )(1 \vee 2\tilde{c_6} \exp(2 \eta_{\max}^2 M^2)).
	\end{align}
	
	So, if $\eta \leq \min\{\frac{1}{m}, \frac{m}{2M^2}\}$ then for any $t \in \mathbb{N}$, the discrete-time algorithm attains the generalization error bound
	\begin{align*}
		|\mathbb{E} \mathrm{gen}(\theta_t)| \leq C_3  \min \left\{\eta t, \frac{n(\eta +2/\alpha)}{(n-k)} \right\}  \cdot \bigg[&\frac{1}{n}\left\{\frac{\eta}{k^{1/2}}  +\eta^{1/2} + k^{1/2}  +  \frac{k}{\eta^{1/2}} \right\} + \left\{\eta  + \frac{\eta}{k^{1/2}}\right\}\bigg],
	\end{align*}
	as required.
\end{proof}

\subsubsection{Convergence of the Bound}\label{appendix:AlphaandEpsilon}

As indicated in Remark~\ref{remark:convergence}, to ensure the convergence of our generalization error bound, we need $\eta \geq \frac{1}{\alpha}\ln{C_1}$, as dictated by the induction process in~\eqref{eq:induction}. We can achieve this through an appropriate choice of $\varepsilon$, as outlined below.

By the definition of $\zeta_r (a)$ in the proof of Lemma~\ref{lemma:2-wass}, $\forall q > 1$,
    \begin{align*}
        a \zeta_r (a) \geq \left[1-r(\frac{q}{q-1})^{\frac{q-1}{q}}a^{-\frac{1}{q}}\right]^+.
    \end{align*}

Therefore, $\forall r < a$,
$$C_1 \leq \frac{1}{\varphi (1-1/a)} \left(1 + \varepsilon \left\{ 2 + 2\sigma_4^{1/2} +2\tilde{c}(2)^{1/2} + \frac{4b}{m} + \frac{2\delta \eta}{km} (d +2) \ell_f^2 \right\}\right).$$

For any $0 \leq \exp(\alpha \eta_{\max} /2) - 1 \leq s < 1,$ let $\varphi = 1-s$ and choose $a$ and $\varepsilon$ as below:
\begin{align}
    a = \left(1 - \frac{e^{-\frac{\alpha \eta}{2}}}{\varphi (1-s)} \right)^{-1} \text{and} 
    \quad \varepsilon = \frac{(1+s) e^{-\frac{\alpha \eta}{4}} - 1}{ 2 + 2\sigma_4^{1/2} +2\tilde{c}(2)^{1/2} + \frac{4b}{m} + \frac{2\delta \eta_{\max}}{km} (d +2) \ell_f^2}.\label{eq:epsilon}
\end{align}

Then, $C_1 \leq (1-s^2)e^{\alpha \eta / 4} \leq e^{\alpha\eta/4} \leq e^{\alpha\eta}$ as required.

\section{SGLD Bounds}

Below, we provide the bounds outlined in \citet{Farghly}, which are subsequently compared with our findings in Table~\ref{table:comparison}.

\begin{lemma}[Moment bound {\citep[Lemma A.1]{Farghly}}]
$$\mu P_t^B (\norm{\cdot}^p) \leq  \mu (\norm{\cdot}^p) + \left[\frac{2b}{m} + 2(p+d-2)/\beta m\right]^{p/2}.$$
\end{lemma}

\begin{lemma}[Moments estimate bound {\citep[Lemma B.2]{Farghly}}]
$$\mu R_\theta^t (\norm{\cdot}^{2p}) \leq \mu (\norm{\cdot}^{2p}) + \tilde{c}(p),$$
		$$\tilde{c}(p) = \frac{1}{m} \left(\frac{6}{m}\right)^{p-1} \left( 1 + \frac{2^{2p} p (2p-1) d }{m\beta} \right) \left[\left(2b + 8 \frac{M^2}{m^2} b \right)^p + 1 + 2 \left(\frac{d}{\beta}\right)^{p-1} (2p-1)^p \right].$$
\end{lemma}
  
\begin{lemma}[Divergence bound {\citep[Lemma B.3]{Farghly}}]
$$\mathbb{E} \norm{\theta_t - \theta_0}^2 \leq 4M^2 \left[\mathbb{E}\norm{\theta_0}^2 + \frac{3b + 2d/\beta}{m}\right] t^2 + 4d\beta^{-1} t.$$
\end{lemma}

\begin{lemma}[Discretization error bound {\citep[Lemma B.4]{Farghly}}]
    $$W_2(\mu R_\theta, \mu R_\Theta)^2 \leq 8 \eta^3  \exp \left( 2 \eta^2 M^2 \right) M^2 ( M^2   \mu(\norm{\cdot}^2) + M^2 b/m + \beta^{-1} d ).$$
\end{lemma}

\end{document}